\documentclass{article}

\PassOptionsToPackage{authoryear, compress}{natbib}

\usepackage{enumitem}
\usepackage{amsmath,amsfonts,amsthm,amssymb,mathrsfs,dsfont,bbm,bm}
\usepackage{mathtools}
\usepackage{amsmath,amssymb,amsthm,amsfonts,latexsym,bm,bbm,xspace,graphicx,float,mathtools,mathdots,physics,mathrsfs,dsfont}
\usepackage{braket,subcaption,ellipsis,textcomp,hhline,pifont,combelow,booktabs,csquotes}
\usepackage[dvipsnames]{xcolor}

\def\eqref#1{Equation~\ref{#1}}

\def\1{\bm{1}}

\def\rvz{{\mathbf{z}}}

\DeclareMathAlphabet{\mathsfit}{\encodingdefault}{\sfdefault}{m}{sl}
\SetMathAlphabet{\mathsfit}{bold}{\encodingdefault}{\sfdefault}{bx}{n}

\def\gB{{\mathcal{B}}}

\def\gF{{\mathcal{F}}}
\def\gG{{\mathcal{G}}}

\def\gL{{\mathcal{L}}}

\def\gN{{\mathcal{N}}}
\def\gO{{\mathcal{O}}}
\def\gP{{\mathcal{P}}}
\def\gQ{{\mathcal{Q}}}

\def\sT{{\mathbb{T}}}

\newtheorem{theorem}{Theorem}
\newtheorem{lemma}{Lemma}

\newtheorem{definition}{Definition}

\newcommand{\numberthis}{\addtocounter{equation}{1}\tag{\theequation}}

\def\R{\mathbb{R}}

\def\N{\mathbb{N}}

\def\Z{\mathbb{Z}}

\def\ll{L^2(\Omega)}

\def\lldf{L^2(\Omega;\R^{d_f})}
\def\lldu{L^2(\Omega;\R^{d_u})}

\def\lln{L^2_N(\Omega)}

\def\llndf{L^2_N(\Omega;\R^{d_f})}
\def\llndu{L^2_N(\Omega;\R^{d_u})}

\def\lldvdv{L^2(\R^{d_v};\R^{d_v})}

\def\h2{H^2(\Omega)}
\def\linf{L^\infty(\Omega)}

\def\h{H_0^1(\Omega)}

\def\project{\Pi_N}

\def\calG{\mathcal{G}}
\def\calF{\mathcal{F}}

\definecolor{prune}{rgb}{0.44, 0.11, 0.11}
\definecolor{myblue}{rgb}{0, .5, 1}
\definecolor{maroon}{rgb}{0.5450, 0, 0}
\definecolor{darkred}{rgb}{0.5450, 0, 0}
\definecolor{RoyalBlue}{RGB}{0,100,170}
\definecolor{DarkBlue}{RGB}{20,70,200}
\definecolor{peach}{rgb}{1, 0.56, 0.56}
\definecolor{NotionGreen}{RGB}{15,123,108}
\definecolor{NotionOrange}{RGB}{217,115,13}
\definecolor{NotionRed}{RGB}{224,62,62}
\definecolor{MontrealBlue}{RGB}{0, 30, 98}

\def\shownotes{1}  %
\ifnum\shownotes=1
\newcommand{\authnote}[2]{{$\ll$\textsf{\footnotesize #1: #2}$\gg$}}
\else
\newcommand{\authnote}[2]{}
\fi

\usepackage[final]{neurips_2023}

\usepackage[utf8]{inputenc} %
\usepackage[T1]{fontenc}    %
\usepackage{multirow}
\usepackage[pagebackref=true]{hyperref} 
\hypersetup{colorlinks, citecolor=purple, linkcolor=purple}
\usepackage{url}            %
\usepackage{booktabs}       %
\usepackage{amsfonts}       %
\usepackage{nicefrac}       %
\usepackage{microtype}      %
\usepackage{xcolor}         %

\newif\ifcomments

\commentsfalse
\commentstrue

\ifcomments
  \newcommand{\colornote}[3]{{\color{#1}\bf{#2: #3}\normalfont}}
\else
  \newcommand{\colornote}[3]{}
\fi

\usepackage[capitalize]{cleveref}
\crefname{equation}{Eq.}{Eq.}
\crefname{figure}{Figure}{Figure~}
\crefname{table}{Table}{Table~}
\crefname{section}{Sec.}{Sec.~}
\crefname{algorithm}{Algorithm}{Algorithm~}
\crefname{thm}{Theorem}{Theorem~}
\crefname{lemma}{Lemma}{Lemma~}
\crefname{appendix}{Appendix}{Appendix~}

\def\ie{\textit{i.e.,~}}

\everypar{\looseness=-1}

\title{Deep Equilibrium Based Neural Operators for Steady-State PDEs}

\author{
Tanya Marwah\footnote{\texttt{tmarwah@andrew.cmu.edu}, Carnegie Mellon University. (Equal Contribution)} 
\and
Ashwini Pokle\footnote{\texttt{apokle@andrew.cmu.edu}, Carnegie Mellon University. (Equal Contribution)} 
\and
J. Zico Kolter\footnote{\texttt{zicokolter@andrew.cmu.edu}, Carnegie Mellon University, Bosch Center for AI} 
\and 
Zachary C. Lipton\footnote{\texttt{zlipton@andrew.cmu.edu}, Carnegie Mellon University.}
\and Jianfeng Lu\footnote{\texttt{jianfeng@math.duke.edu}, Duke University.} 
\and Andrej Risteski\footnote{\texttt{aristesk@andrew.cmu.edu}, Carnegie Mellon University. 
}}

\author{
  Tanya Marwah$^1$\thanks{Equal contribution. Correspondence to \texttt{tmarwah@andrew.cmu.edu} and \texttt{apokle@andrew.cmu.edu}} \quad 
  Ashwini Pokle $^{1*}$ \quad
  J. Zico Kolter $^{1, 2}$ \quad 
  \textbf{Zachary C. Lipton} $^1$ \quad \\
  \textbf{Jianfeng Lu} $^3$ \quad 
  \textbf{Andrej Risteski} $^1$ \\
  $^1$Carnegie Mellon University \quad  $^2$ Bosch Center for AI \quad $^3$ Duke University \\
  \texttt{\{tmarwah,apokle,zicokolter,zlipton,aristesk\}@andrew.cmu.edu} \quad \\ \texttt{jianfeng@math.duke.edu}
  }

\begin{document}

\maketitle

\begin{abstract}
Data-driven machine learning approaches 
are being increasingly used to solve partial differential equations (PDEs). They have shown particularly striking successes when training an operator, which takes as input a PDE in some family, and outputs its solution.   
However, the architectural design space, especially given structural knowledge of the PDE family of interest, is still poorly understood.    
We seek to remedy this gap by studying the benefits of weight-tied neural network architectures for steady-state PDEs. 
To achieve this, we first demonstrate that the solution of most steady-state PDEs can be expressed as a fixed point of a non-linear operator. 
Motivated by this observation, we propose FNO-DEQ, a deep equilibrium variant of the FNO architecture that directly solves for the solution of a steady-state PDE as the infinite-depth fixed point of an implicit operator layer using a black-box root solver and differentiates analytically through this fixed point resulting in $\mathcal{O}(1)$ training memory.
Our experiments indicate that FNO-DEQ-based architectures outperform 
FNO-based baselines with $4\times$ the number of parameters in predicting the solution to steady-state PDEs such as Darcy Flow and steady-state incompressible Navier-Stokes.
Finally, we show FNO-DEQ 
is more robust 
when trained with datasets with more noisy observations
than the FNO-based baselines, 
demonstrating the benefits of using appropriate inductive biases in architectural design for different 
neural network based PDE solvers.
Further, we 
show a universal approximation result that demonstrates that FNO-DEQ can approximate the solution to any steady-state PDE that can be written as a fixed point equation.
\end{abstract}
\section{Introduction}
Partial differential equations (PDEs) 
are used to model a wide range of processes in science and engineering.
They define a relationship of (unknown) function and its partial derivatives. %
Most PDEs do not admit a closed form solution, and are solved using
a variety of classical numerical methods
such as finite element~\citep{leveque2007finite},
finite volume~\citep{moukalled2016finite}, and spectral methods~\citep{kopriva2009implementing, boyd2001chebyshev}. These methods are often very computationally expensive, both as the ambient dimension grows, and as the desired accuracy increases.

This has motivated a rapidly growing area of research in data-driven approaches to PDE solving. One promising approach involves  learning \emph{neural solution operators}~\citep{chen1995universal, lu2019deeponet, bhattacharya2021model, li2020neural}, which take in the coefficients of a PDE in some family and output its solution---and are trained by examples of coefficient-solution pairs.  

While several architectures for this task have been proposed, the design space---in particular taking into account structural properties of the PDEs the operator is trained on---is still largely unexplored. Most present architectures are based on ``neuralizing'' a classical numerical method. For instance, \cite{li2020fourier} take inspiration from spectral methods, and introduce FNO: a trained composition of (parametrized) kernels in Fourier space. \cite{brandstetter2022message} instead consider 
finite-difference methods and generalize them into (learnable) graph neural networks using message-passing.  

Our work focuses on families of PDEs that describe the steady-state of a system (that is, there is no time variable). Namely, we consider equations of the form: %

\begin{equation}
    \label{eq:lu_f}
    L(a(x), u(x)) = f(x), \qquad \forall x \in \Omega,
\end{equation}
where $u:\Omega \to \R^{d_u}$, 
$a: \Omega \to \R^{d_a}$ 
and $f: \Omega \to \R^{d_f}$ are functions defined over the domain $\Omega$, 
and $L$ is a (possibly non-linear) operator. 
This family includes many natural PDE families like Poisson equations, electrostatic equations, and steady-state Navier-Stokes.

We take inspiration from classical numerical approaches of fast-converging Newton-like iterative schemes 
\citep{leveque2007finite, farago2002numerical}
to solve steady-state PDEs, as well as recent theoretical works for elliptic (linear and non-linear PDEs)~\citep{marwah2021parametric, chen2021representation, marwah2022neural} to hypothesize that very deep, but heavily weight-tied architectures would provide a useful 
architectural design choice for steady-state PDEs.

In this paper, we show that for steady state equations
it is often more beneficial to weight-tie
an existing neural operator, 
as opposed to making the model deeper---thus increasing its size.
To this end, we introduce {\bf FNO-DEQ}, a new architecture for solving steady-state PDEs. FNO-DEQ 
is a deep equilibrium model (DEQ)
that utilizes weight-tied FNO layers along with 
implicit differentiation and root-solvers to approximate the solution of a steady-state PDE.
DEQs are a perfect match to the desiderata laid out above: they can be viewed alternately 
as directly parameterizing the fixed points 
of some iterative process; or
by explicitly expanding some iterative fixed point solver like 
Newton's or Broyden's method as an infinitely deep, weight-tied model.

Such an architecture has a distinct computational advantage: %
implicit layer models effectively
backpropagate through the infinite-depth network while using only constant memory (equivalent to a single layer’s activations).
Empirically, we show that for steady-state PDEs, 
weight-tied and DEQ based models perform better than baselines with 4$\times$ 
the number of parameters, and are robust to training data noise. 
In summary, we make the following contributions:
\begin{itemize}%
    \item We show the benefits of weight-tying as an effective architectural choice for neural operators when applied to steady-state PDEs.
    \item We introduce FNO-DEQ, a FNO based deep equilibrium model (DEQ) that uses implicit layers and 
        root solving to approximate the solution of a steady-state PDE. We further attest to the empirical performance 
        of FNO-DEQ by showing that it performs as well as FNO and its variants with $4\times$ number of parameters.
    \item  We show that FNO-DEQ and weight tied architectures 
    are more robust to both input and observation noise, 
    thus showing that weight-tying is a useful inductive bias for 
    architectural design for steady-state PDEs.
    \item By leveraging the universal approximation results of FNO~\citep{kovachki2021universal} we show that 
        FNO-DEQ based architectures can universally approximate the solution operator for a wide variety of steady-state PDE families. 
    \item Finally, we create a dataset of pairs of steady-state incompressible Navier-Stokes equations with different forcing functions and viscosities, along with their solutions, which we will make public as a community benchmark for steady-state PDE solvers. 
\end{itemize}

\section{Related Work}
\label{sec:related_work}

Neural network based approaches for solving PDEs 
can broadly be divided into two categories. 
First are hybrid solvers~\citep{bar2019learning, kochkov2021machine, hsieh2019learning}
which use neural networks in conjunction with existing numerical solvers. The main motivation is to 
not only improve upon the existing solvers, but to 
also replace the more computationally inefficient parts of the solver
with a learned counter part.
Second set of approaches are 
full machine learning based approaches
that aim 
to leverage the 
approximation capabilities of neural networks~\citep{hornik1989multilayer}
to directly learn the dynamics of the physical system from observations.

Hybrid solvers like \citet{hsieh2019learning} use a neural network to learn a correction term 
to correct over an existing hand designed solver for a Poisson equation, and
also provide convergence guarantees of their method to the solution of the PDE.  
However, the experiments in their paper are limited to linear elliptic PDEs.
Further, solvers like~\citet{bar2019learning} use neural networks to derive the discretizations for a given PDE, thus enabling 
the use of a low-resolution grid in the numerical solver. 
Furthermore,~\citet{kochkov2021machine} use neural networks to interpolate differential operators between grid points
of a low-resolution grid with high accuracy. 
This work specifically focuses on solving Navier-Stokes equations, 
their method is more accurate than numerical techniques like Direct Numerical Simulation (DNS) 
with a low-resolution grid, and is also $80\times$ more faster.
~\citet{brandstetter2022message} introduced a message passing based hybrid scheme 
to train a hybrid solver and also propose a loss term which helps improve the stability of hybrid solvers for time dependent PDEs.
However, most of these methods are equation specific, and are not easily
transferable to other PDEs 
from the same family.

The neural network based approach that has recently 
garnered the most interest by the community is that of 
the operator learning framework
~\citep{chen1995universal,kovachki2021neural,lu2019deeponet,li2020fourier,bhattacharya2021model},
which uses a neural network to approximate 
and infinite dimensional operator between two Banach spaces, 
thus learning an entire family of PDEs at once.
~\citet{lu2019deeponet} introduces DeepONet, which uses two 
deep neural networks, referred to as the branch net and trunk net, which are 
trained concurrently to learn from data.
Another line of operator learning framework is that of neural operators~\cite{kovachki2021neural}.
The most successful methodology for neural operators being the
Fourier neural operators (FNO)
~\citep{li2020fourier}.
FNO 
uses convolution based integral kernels which are evaluated in the
Fourier space. 
Future works like~\cite{tran2021factorized}
introduce architectural improvements that enables one to train 
deeper FNO networks, thus increasing their size and improving their 
the performance on a variety of (time-dependent) PDEs. 
Moreover, the success of Transformers in domains like language and vision
has also inspired transformer based neural operators in works like
\citet{li2022transformer, hao2023gnot} and \citet{Liu2022-nd}.
Theoretical results pertaining to the neural operators
mostly include universal approximation results~\cite{kovachki2021universal,lanthaler2022error} which show
that architectures like FNO and DeepONet can indeed approximate the infinite dimension operators.

In this work, we focus on steady-state equations
and show the benefits of weight-tying in improving the performance 
of FNO for steady-state equations. 
We show that instead of making a network deeper and hence increasing the size
of a network, 
weight-tied FNO architectures can outperform FNO and its variants $4\times$
its size. 
We further introduce FNO-DEQ, a deep equilibrium 
based architecture to simulate an infinitely deep weight-tied 
network (by solving for a fixed point)
with $\mathcal{O}(1)$ training memory.
Our work takes inspiration from 
recent 
theoretical works like~\cite{marwah2021parametric, chen2021representation, marwah2022neural}
which derive parametric rates for some-steady state equations, and 
in fact prove that neural networks can approximate solutions to 
some families of PDEs with just $\mbox{poly}(d)$ parameters, 
thus evading the curse of dimensionality.

\section{Preliminaries}
\label{sec:background}
We now introduce some key concepts and notation. 
\begin{definition}[$L^2(\Omega; \R^d)$]
    For a domain $\Omega$ we denote by $L^2(\Omega; \R^d)$ 
    the space of square integrable functions $g: \Omega \to \R^d$
    such that 
    $\|g\|_{\ll} < \infty$, where
    $\|g\|_{\ll} = \left(\int_\Omega \|g(x)\|_{\ell_2}^2 dx\right)^{1/2}$.
\end{definition}

\subsection{Neural Operators}
Neural operators \citep{lu2019deeponet, li2020fourier, bhattacharya2021model, patel2021physics, kovachki2023neural} 
are a deep learning approach to learning solution operators which map a PDE to its solution. 
Fourier Neural Operator (FNO) \citep{li2020fourier} is a particularly successful recent architecture 
parametrized as a sequence of kernel integral operator layers followed by non-linear activation functions. Each kernel integral operator layer is a convolution-based kernel function that is instantiated through a linear transformation in Fourier domain, making it less sensitive to the level of spatial discretization. Specifically, an $L$-layered FNO 
$G_\theta: \R^{d_u} \to \R^{d_u}$ 
with learnable parameters $\theta$, is defined as 
\begin{equation}
    \label{eq:fno_layer_def}
    G_\theta := \gQ \circ \mathcal{L}_L \circ \mathcal{L}_{L-1} \circ \cdots \circ \mathcal{L}_1 \circ \gP
\end{equation}
where 
$\mathcal{P}: \lldu \to \lldvdv$ 
and 
$\mathcal{Q}: L^2(\R^{d_v};\R^{d_v}) \to  
L^{2}(\R^{d_v};\R^{d_u}) $ 
are projection operators, 
and 
$\mathcal{L}_{l}: L^2(\R^{d_v};\R^{d_v})\to L^2(\R^{d_v};\R^{d_v})$
for $l \in [L]$ is the $l^{\text{th}}$ FNO layer defined as, 
\begin{equation}
    \label{eq:fno_layer}
    \mathcal{L}_{l}\left(v_{l}\right) 
    = \sigma \left(W_{l} v_{l} + b_{l} + \mathcal{K}_{l}(v_l))\right).
\end{equation}
Here $\sigma$ is a non-linear activation function, $W_l, b_l$ are the $l^{th}$ layer weight matrix and bias terms. Finally
$\mathcal{K}_{l}$ 
is the $l^{th}$ integral kernel operator 
which is calculated using the Fourier transform as introduced in ~\citet{li2020fourier}
defined as follows,
\begin{equation}
    \label{eq:kernal_integral_operator}
    \mathcal{K}_l(v_l) = \mathcal{F}^{-1} \left(R_l \cdot \left(\mathcal{F} v_l\right)\right)(x) \qquad \forall x \in \Omega, 
\end{equation}
where $\calF$ and $\calF^{-1}$ are the 
Fourier transform and the inverse 
Fourier transform, with $R_l$
representing the learnable weight-matrix in the Fourier domain. 
Therefore, ultimately, the trainable parameters $\theta$ is a collection of all the weight matrices and biases, i.e, 
$\theta := \{W_l, b_l, R_l, \cdots, W_1, b_1, R_1\}$.

\subsection{Equilibrium Models}
Equilibrium models \citep{liao2018reviving, bai2019deep, revay2020lipschitz, winston2020monotone} compute internal representations by solving for a fixed point in their forward pass. 
Specifically, consider a deep feedforward network with $L$ layers :
\begin{equation}
    z^{[i+1]} = f_\theta^{[i]}\left(z^{[i]}; x \right) \quad \text{for} \; i = 0, ..., L-1
\end{equation}
where $x \in \mathbb{R}^{n_x}$ is the input injection, $z^{[i]} \in \mathbb{R}^{n_z}$ is the hidden state of $i^{th}$ layer with $z^{[0]} = \mathbf{0}$, and $f_\theta^{[i]} : \mathbb{R}^{n_x \times n_z} \mapsto \mathbb{R}^{n_z}$ is the feature transformation of $i^{th}$ layer, parametrized by $\theta$. 
Suppose the above model is weight-tied, \ie  $f_\theta^{[i]} = f_\theta, \forall i$, and $\lim_{i \rightarrow \infty} f_\theta \left( z^{[i]}; x \right)$ exists and its value is $z^\star$. Further, assume that for this $z^\star$, it holds that $f_\theta\left(z^\star; x \right) = z^\star$.  Then, equilibrium models can be interpreted as the infinite-depth limit of the above network such that
$f^\infty_\theta\left(z^\star; x \right) = z^\star $

Under certain conditions\footnote{The fixed point can be reached if the  dynamical system is globally contractive. This is usually not true in practice for most choices of $f_\theta$, and divergence is possible.}, and for certain classes of $f_\theta$\footnote{\citet{bai2019deep} state that $f_\theta$ needs to be stable and constrained. In general, by Banach's fixed point theorem, global convergence is guaranteed if $f_\theta$ is contractive over its input domain.}, the output $z^\star$ of the above weight-tied network is a fixed point. 
A simple way to solve for this fixed point is to use fixed point iterations, \ie repeatedly apply the update $z^{[t+1]} = f_\theta(z^{[t]}; x)$ some fixed number of times, and backpropagate through the network to compute gradients. However, this can be computationally expensive.
Deep equilibrium (DEQ) models \citep{bai2019deep} 
explicitly solve 
for $z^\star$
through iterative root finding methods like Broyden's method~\citep{broyden1965class}, Newton's method, Anderson acceleration~\citep{anderson1965iterative}. DEQs use implicit function theorem to directly differentiate through the fixed point $z^\star$ at equilibrium, thus requiring constant memory to backpropagate through an infinite-depth network:
\begin{equation}
    \dfrac{\partial z^\star}{\partial \theta} =  \left( I - \dfrac{\partial f_\theta(z^\star; x)}{\partial z^\star}\right)^{-1} \dfrac{\partial f_\theta (z^\star; x)}{\partial \theta} \label{eq:implcit-grad-deq}
\end{equation}
Computing the inverse of Jacobian can quickly become intractable as we deal with high-dimensional feature maps. One can replace the inverse-Jacobian term with an identity matrix \ie Jacobian-free \citep{fung2022jfb} or an approximate inverse-Jacobian \citep{geng2021training} without affecting the final performance. There are alternate formulations of DEQs~\citep{winston2020monotone} that guarantee existence of a unique equilibrium point. However, designing $f_\theta$ for these formulations can be challenging, and in this work we use the formulation by \citet{bai2019deep}.

\section{Problem setting}

We first formally define the system of steady-state PDEs that we will solve for:
\begin{definition}[Steady-State PDE] \label{def:steady_state_PDE}
    Given a bounded open set $\Omega \subset \R^d$, 
    a steady-state PDE 
    can be written in the following general form:
    \begin{equation}
        \begin{split}
            L(a(x), u(x)) &= f(x), \qquad \forall x \in \Omega \\
        \end{split}
    \end{equation}
    Here $L$ is a continuous operator, the function
    $u \in L^2\left(\Omega; \R^{d_u}\right)$
    is the unknown function that we wish to solve for
    and 
    $a \in L^2\left(\Omega; \R^{d_a}\right)$
    collects all the coefficients describing the PDE, and 
    $f \in L^2\left(\Omega; \R^{d_f}\right)$
    is a function independent of $u$.
    We will, for concreteness, assume periodic boundary conditions, i.e. $\forall z \in \Z^d, \forall x \in \Omega$ we have $u(x + z) = u(x)$. (Equivalently,  
    $\Omega:= \sT^d = [0, 2\pi]^d$ can be identified with the torus.)
    \footnote{
   This is for convenience of exposition, our methods can readily be extended to other boundary conditions 
    like Dirichet, Neumann etc. 
    }
    Finally, we will denote $u^\star: \Omega \to \R$ as the solution 
    to the PDE.
\end{definition}
Steady-state models a system at stationarity, \ie when some quantity of interest like temperature or velocity no longer changes over time. Classical numerical solvers for these PDEs include iterative methods like Newton updates
or conjugate gradient descent, typically with carefully chosen preconditioning to ensure benign conditioning and fast convergence.   
Furthermore, recent theoretical works \citep{marwah2021parametric,chen2021representation,marwah2022neural} 
show that for many families of PDEs 
(e.g., steady-state elliptic PDEs that admit a variational formulation), 
iterative algorithms can be efficiently ``neuralized'', 
that is, the iterative algorithm can be represented by a compact neural network, 
so long as the coefficients of the PDE are also representable by a compact neural network.
Moreover, the architectures constructed in these works are heavily weight-tied.

We will operationalize these developments through the
additional observation that all these iterative schemes can be viewed as algorithms to find a fixed point of a suitably chosen operator. Namely, we can design an operator
$\gG: L^2(\Omega; \R^{d_u}) \times L^2(\Omega; \R^{d_f}) \to L^2(\Omega; \R^{d_u})$ 
\footnote{We note that the choice of defining the operator
with the forcing function $f$ is made for purely expository purposes
the operator $\gG$ can be defined as 
$\gG: \lldu \times L^2(\Omega; \R^{d_a}) \to \lldu$ as well.}
such that $u^\star = \gG(u^\star, f)$ and a lot of common (preconditioned) first and second-order methods are natural ways to recover the fixed points $u^\star$. 

Before describing our architectures, we introduce two components that we will repeatedly use. 

\begin{definition}[Projection and embedding layers] 
    A projection and embedding layer, respectively 
    $\gP: \lldu \times \lldf \to \lldvdv \times \lldvdv$
    and 
    $\gQ: \lldvdv \to L^2(\R^{d_v}; \R^{d_u})$, 
    are defined as
    \begin{align}
        \gP(v, f) 
        &= \left(\sigma\left(W_P^{(1)}v + b_{P}^{(1)} \right),
            \sigma\left(W_P^{(2)}f + b_P^{(2)}\right)\right)\nonumber, 
            \\
        \gQ(v) 
        &= \sigma\left(W_Qv + b_Q \right) \nonumber
    \end{align}
    where 
    $W_P^{(1)} \in \R^{d_u \times d_v}, W_P^{(2)} \in \R^{d_f \times d_v}, W_Q \in \R^{d_v\times d_u}$ 
    and 
    $b_P^{(1)},b_P^{(2)} \in \R^{d_v}, b_Q \in \R^{d_u}$.
\label{d:projection}
\end{definition}

\begin{definition}[Input-injected FNO layer]
\label{d:inputfno}
An input-injected FNO layer 
    $\gL: \lldvdv \times \lldvdv \to \lldvdv$ is defined as 
    \begin{equation}
        \label{eq:fno_layer_updated}
       \gL(v, g) := g + \sigma
            \left(W v + b + \gF^{-1}(R^{(k)} \cdot (\gF v)\right).
    \end{equation}
    where $W \in \R^{d_v \times d_v}$, $b \in \R^{d_v}$
    and $R^{(k)} \in \R^{d_v \times d_v}$ for all $k \in [K]$
    are learnable parameters. 
\end{definition}
Note the difference between the FNO layer specified above, and the standard FNO layer \eqref{eq:fno_layer} is the extra input $g$ to the layer, 
which in our architecture will correspond to a projection of (some or all) of the PDE coefficients.
We also note that this change to the FNO layer also enables us to learn deeper FNO architectures, as
shown in Section~\ref{sec:experiments}.
With this in mind, we can discuss the architectures we propose.

\paragraph{Weight-tied architecture I: Weight-tied FNO}
The first architecture we consider is a weight-tied version of FNO (introduced in Section~\ref{sec:background}), in which we repeatedly apply ($M$ times) the same transformation in each layer. More precisely, we have:

\begin{definition}[FNO Weight-Tied]
    \label{def:fno_weight_tied}
    We define a $M$ times weight-tied neural operator 
    $G^M_\theta$ as,
    \begin{align*}
        G^M_\theta = \gQ \circ 
        \underbrace{\gB^L \circ \gB^L \circ \cdots \circ \gB^L}_{\text{M times}} \circ \gP
    \end{align*}
    such that: $\gP, \gQ$ are projection and embedding layers as in Definition~\ref{d:projection} 
        an $\gB^L: \lldvdv \times \lldvdv \to \lldvdv$
    is an $L$-layer FNO block, i.e, $
        \gB^L = \gL_{L} \circ \gL_{L-1} \circ \gL_{L-2} \circ \gL_{1}$ 
        where for all $l \in [L]$, $\gL_l(\cdot, \gP(f))$
        \footnote{We are abusing the notation somewhat and denoting by $\gP(f)$ the second coordinate of $\gP$, which only depends on $f$.} is an input-injected FNO block as in Definition~\ref{d:inputfno}. 
    
\end{definition}

\paragraph{Weight-tied architecture II: FNO-DEQ} %
In cases where we believe a weight-tied $G_\theta^M$ converges to some fixed point as $M \to \infty$, 
unrolling $G_\theta^M$ 
for a large $M$ requires a lot of hardware memory for training:  training the model requires one to store intermediate hidden units 
for each weight-tied layer for backpropagation, resulting in a $\gO(M)$
increase in the amount of memory required. 

To this end, we use Deep Equilibrium models (DEQs) which enables 
us to implicitly train $G_\theta := \lim_{M\to\infty} G_\theta^{M}$ by directly 
solving for the fixed point by leveraging black-box root finding algorithms
like quasi-Newton methods,
~\citep{broyden1965class,anderson1965iterative}. Therefore
we can think of this approach as an infinite-depth (or infinitely unrolled) weight-tied network. 
We refer to this architecture as \textbf{FNO-DEQ}. 
\begin{definition}[FNO-DEQ]
    \label{def:fno_deq}
    Given $\gP, \gQ$ and $\gB^{L}$ in 
    Definition~\ref{def:fno_weight_tied}, 
    
    FNO-DEQ is trained to parametrize the fixed point equation 
    $\gB^{L}\left(v^\star, \gP(f)\right) = v^\star$ and outputs $u^\star = \gQ(v^\star)$.
\end{definition}

Usually, it is non-trivial to differentiate through these black-box root solvers. DEQs enable us to implicitly differentiate through the equilibrium 
fixed point efficiently without any need to backpropagate through these root solvers, therefore resulting in $\gO(1)$ training memory.

\section{Experiments}
\label{sec:experiments}

\textbf{Network architectures.} We consider the following network architectures in our experiments.

\textbf{FNO}: We closely follow the architecture proposed by \citet{li2020fourier} and construct this network by stacking four FNO layers and four convolutional layers, separated by GELU activation \citep{gelu}. Note that in our current set up, we recover the original FNO architecture if the input to the $l^{\text{th}}$ layer is the output of $(l-1)^{\text{th}}$ layer \ie 
$v_l = \gB_{l-1}(v_{l-1})$.

\textbf{Improved FNO (FNO++
)}: 
The original FNO architecture suffers from vanishing gradients, which prohibits it from being made deeper~\citep{tran2021factorized}. We overcome this limitation by introducing residual connections within each block of FNO, with each FNO block $\gB_l$ comprising of three FNO layers $\gL$ \ie $\gB_l = \gL^l_{L_1} \circ \gL^l_{L_2} \circ \gL^l_{L_3}$ and three convolutional layers, where $\gL$ is defined in \cref{eq:fno_layer_updated}.

\textbf{Weight-tied network (FNO-WT)}: 
This is the weight-tied architecture introduced in Definition~\ref{def:fno_weight_tied},
where we initialize $v_0(x) = 0$ for all $x \in \Omega$.

\textbf{FNO-DEQ}: As introduced in Definition~\ref{def:fno_deq}, FNO-DEQ is a
weight-tied network where we explicitly solve for the fixed point in the forward pass with a root finding algorithm. 
We use Anderson acceleration \citep{anderson1965iterative} as the root solver. 
For the backward pass, we use approximate implicit gradients \citep{geng2021training} which are light-weight and more stable in practice, compared to implicit gradients computed by inverting Jacobian. 

Note that both weight-tied networks and FNO-DEQs leverage weight-tying but the two models differ in the ultimate goal of the forward pass: DEQs explicitly solve for the fixed point during the forward pass, while weight-tied networks trained with backpropagation may or may-not reach a fixed point \citep{anil2022path}. 
Furthermore, DEQs require $\mathcal{O}(1)$ memory, as they differentiate through the fixed point implicitly, whereas weight-tied networks need to explicitly create the entire computation graph for backpropagation, which can become very large as the network depth increases.
Additionally, FNO++ serves as a non weight-tied counterpart to a weight-tied input-injected network. Like weight-tied networks, FNO++ does not aim to solve for a fixed point in the forward pass.

\textbf{Experimental setup.} We test the aforementioned network architectures on two families of steady-state PDEs: Darcy Flow equation and steady-state Navier-Stokes equation for incompressible fluids. 
For experiments with Darcy Flow, we use the dataset provided by \cite{li2020fourier}, 
and generate our own dataset for steady-state Navier-Stokes.
For more details on the datasets and the data generation processes we refer to 
Sections~\ref{subsec:Darcy_flow_implementation} 
and ~\ref{subsec:navier_stokes_implementation} of the Appendix.
For each family of PDE, we train networks under 3 different training setups: clean data, noisy inputs and noisy observations. For experiments with noisy data, both input and observations, we add noise sampled from a sequence of standard Gaussians with increasing values of variance $\{ \gN(0, (\sigma_k^2))\}_{k=0}^{M-1}$, 
where $M$ is the total number of Gaussians we sample from. 
We set $\sigma^2_0 = 0$ and $\sigma^2_{\text{max}} = \sigma^2_{M-1} \leq 1/r$, where $r$ is the resolution of the grid. Thus, the training data includes equal number of PDEs with different levels of Gaussian noise added to their input or observations.
We add noise to training data, and always test on clean data. We follow prior work \citep{li2020neural} and report relative $L_2$ norm between ground truth $u^\star$ 
and prediction on test data.
The total depth of all networks besides FNO is given by $6B + 4$, where $B$ is the number of FNO blocks. Each FNO block has 3 FNO layers and convolutional layers. 
In addition, we include the depth due to $\mathcal{P}$, $\mathcal{Q}$, and an additional final FNO layer and a convolutional layer. 
We further elaborate upon network architectures and loss functions in 
in \cref{sec:implementation_details}.

\vspace{-2mm}
\subsection{Darcy Flow}
\label{subsec:darcy_flow}
For our first set of experiments we consider stationary Darcy Flow equations, a form of linear elliptic PDE with 
in two dimensions. 
The PDE is defined as follows,
\begin{equation}
    \begin{split}
        -\nabla \cdot (a(x)\nabla u(x)) &= f(x), \qquad x \in (0, 1)^2 \\
        u(x) &= 0 \qquad\qquad x \in \partial(0, 1)^2.
    \end{split}
\end{equation}
Note that the diffusion coefficient
$a \in \linf(\Omega; \R_+)$, 
i.e., the coefficients are always positive, 
and 
$f\in
L^2(\Omega; \R^{d_f})$
is the forcing term.
These PDEs are used to model the 
steady-state pressure of fluids flowing through a porous media. 
They can also be used to model the stationary state of the diffusive process with $u(x)$ modeling 
the temperature distribution through the space with $a$ defining the thermal conductivity of the medium.
The task is to learn an operator 
$G_\theta: L^2(\Omega; \R^{d_u}) \times L^2(\Omega; \R^{d_a}) \to L^2(\Omega; \R^{d_u})$
such that $u^\star = G_\theta(u^\star, a)$.

We report the results of our experiments on Darcy Flow in \cref{table:results-darcy-flow-all}. 
The original FNO architecture does not improve its performance  with increased number of FNO blocks $\gB$.
FNO++ with residual connections scales better but saturates at around 4 FNO blocks. In contrast, FNO-WT and FNO-DEQ with just a \emph{single} FNO block outperform deeper non-weight-tied architectures on clean data and on data with noisy inputs. 
When  observations are noisy, FNO-WT and FNO-DEQ outperform FNO++ with a similar number of parameters, and perform comparably to FNO++ with $4\times$ parameters. 

We also report results on shallow FNO-DEQ, FNO-WT and FNO++ architectures. An FNO block in these shallow networks has a single FNO layer instead of three layers. In our experiments, shallow weight-tied networks outperform non-weight-tied architectures including FNO++ with $7\times$ parameters on clean data and on data with noisy inputs, and  perform comparably to deep FNO++ on noisy observations. In case of noisy observations, we encounter training instability issues in FNO-DEQ. We believe that this shallow network lacks sufficient representation power and cannot accurately solve for the fixed point during the forward pass. These errors in fixed point estimation accumulate over time, leading to incorrect values of implicit gradients, which in turn result in training instability issues.
\begin{table*}[th!]
\centering
\resizebox{0.95\textwidth}{!}{%
\begin{tabular}{cccccc}
\toprule
\multirow{3}{*}{Architecture} & \multirow{3}{*}{Parameters} & \multirow{3}{*}{\#Blocks} & \multicolumn{3}{c}{Test error $\downarrow$} \\
\cmidrule(lr){4-6}
& & & $\sigma^2_{\max}=0$ & $(\sigma^2_{\max})^i=0.001$ & $(\sigma^2_{\max})^t=0.001$ \\
\midrule
FNO & 2.37M & 1 & 0.0080 $\pm$ 5e-4 & 0.0079 $\pm$ 2e-4  &  0.0125 $\pm$ 4e-5 \\
FNO & 4.15M & 2 & 0.0105 $\pm$ 6e-4 & 0.0106 $\pm$ 4e-4 & 0.0136 $\pm$ 2e-5 \\
FNO & 7.71M & 4 & 0.2550 $\pm$ 2e-8 & 0.2557 $\pm$ 8e-9 & 0.2617 $\pm$ 2e-9 \\
\midrule
FNO++ & 2.37M & 1 & 0.0075 $\pm$ 2e-4 & 0.0075 $\pm$ 2e-4 &  0.0145 $\pm$ 7e-4 \\
FNO++ & 4.15M & 2 & 0.0065 $\pm$ 2e-4 & 0.0065 $\pm$ 9e-5 & 0.0117 $\pm$ 5e-5 \\
FNO++ & 7.71M & 4 & 0.0064 $\pm$ 2e-4 & 0.0064 $\pm$ 2e-4 & \textbf{0.0109 $\pm$ 5e-4}  \\
S-FNO++ & 1.78M & 0.66 & 0.0093 $\pm$ 5e-4 & 0.0094 $\pm$ 7e-4  & 0.0402 $\pm$ 6e-3 \\
\midrule
FNO-WT & 2.37M & 1 & \textbf{0.0055 $\pm$ 1e-4} & \textbf{0.0056 $\pm$ 5e-5} & 0.0112 $\pm$ 4e-4 \\
FNO-DEQ & 2.37M &  1 & \textbf{0.0055 $\pm$ 1e-4} & \textbf{0.0056 $\pm$ 7e-5} & 0.0112 $\pm$ 4e-4 \\
\midrule
S-FNO-WT 
& 1.19M & 0.33 & 0.0057 $\pm$ 3e-5 & 0.0057 $\pm$ 5e-5 & 0.0112 $\pm$ 1e-4 \\
S-FNO-DEQ 
& 1.19M & 0.33 & 0.0056 $\pm$ 4e-5 & 0.0056 $\pm$ 5e-5 & 0.0136 $\pm$ 0.011 \\
\bottomrule
\end{tabular}}
\caption{Results on Darcy flow: clean data (Col 4),noisy inputs (Col 5) and noisy observations (Col 6) with max variance of added noise being $(\sigma^2_{\max})^i$ and $(\sigma^2_{\max})^t$, respectively. Reported test error has been averaged on three different runs with seeds 0, 1, and 2.
Here, S-FNO++, S-FNO-WT and S-FNO-DEQ are shallow versions 
of FNO++, FNO-WT and FNO-DEQ respectively.
}
\label{table:results-darcy-flow-all}
\end{table*}
\vspace{-2mm}

\subsection{Steady-state Navier-Stokes Equations for Incompressible Flow}
\label{subsec:navier_stokes}
We consider the steady-state Navier-Stokes equation for an incompressible viscous fluid in the vorticity form
defined on a torus, i.e., with periodic boundary condition,
\begin{equation}
    \begin{split}
        \label{eq:navier_stokes}
        u \cdot \nabla \omega &= \nu \Delta \omega + f, \qquad x \in \Omega\\
        \nabla \cdot u &= 0 \qquad\qquad\quad\;\;  x \in \Omega
    \end{split}
\end{equation}
where $\Omega:=(0, 2\pi)^2$, and $u:\Omega \to \R^2$ is the velocity and 
$\omega:\Omega \to \R$ where $\omega = \nabla \times u$, $\nu \in \R_+$ is the viscosity 
and $f: \Omega \to \R$ is the external force term.
We learn an operator $G_\theta: L^2(\Omega; \R^{d_u}) \times L^2(\Omega; \R^{d_f}) \to L^2(\Omega; \R^{d_u})$, such that 
$u^\star = G_\theta(u^\star, f)$.
We train all the models on data with viscosity values $\nu = 0.01$ and $\nu=0.001$, and create a dataset
for steady-state incompressible
Navier-Stokes, which we will make public as a community benchmark
for steady-state PDE solvers.

Results for Navier-Stokes equation have been reported in \cref{table:results-navier-stokes-visc-0.001-nl-0.001} and \cref{table:results-navier-stokes-visc-0.01-nl-0.001}. For both values of viscosity, FNO-DEQ outperforms other architectures for all three cases: clean data, noisy inputs and noisy observations. 
FNO-DEQ is more robust to noisy inputs compared to non-weight-tied architectures. 
For noisy inputs, FNO-DEQ matches the test-error of noiseless case in case of viscosity $0.01$ and almost matches the test-error of noiseless case in case of viscosity $0.001$. 
We provide additional results for noise level $0.004$ in Appendix~\ref{sec:additional-experiments}. %
FNO-DEQ and FNO-WT consistently outperform non-weight-tied architectures for higher levels of noise as well.

In general, DEQ-based architectures are slower to train  (upto $\sim$2.5$\times$ compared to feedforward networks of similar size) as we solve for the fixed point in the forward pass. 
However, their inductive bias provides performance gains that cannot be achieved by simply stacking non-weight-tied FNO layers. In general, we observe diminishing returns in FNO++ beyond 4 blocks. Additionally, training the original FNO network on more than 4 FNO blocks is challenging, with the network often diverging during training, and therefore we do not include these results in the paper.

\begin{table*}[th!]
\centering
\resizebox{0.95\textwidth}{!}{%
\begin{tabular}{cccccc}
\toprule
\multirow{3}{*}{Architecture} & \multirow{3}{*}{Parameters} & \multirow{3}{*}{\#Blocks} & \multicolumn{3}{c}{Test error $\downarrow$} \\
\cmidrule(lr){4-6}
& & & $\sigma^2_{\max}=0$ & $(\sigma^2_{\max})^i=0.001$ & $(\sigma^2_{\max})^t=0.001$ \\
\midrule
FNO & 2.37M & 1 & 0.184 $\pm$ 0.002 & 0.218 $\pm$ 0.003 & 0.184 $\pm$ 0.001 \\
FNO & 4.15M & 2 & 0.162 $\pm$ 0.024 & 0.176 $\pm$ 0.004 & 0.152 $\pm$ 0.005 \\
FNO & 7.71M & 4 & 0.157 $\pm$ 0.012 & 0.187 $\pm$ 0.004 & 0.166 $\pm$ 0.013 \\
\midrule
FNO++ & 2.37M & 1 & 0.199 $\pm$ 0.001 & 0.230 $\pm$ 0.001 & 0.197 $\pm$ 0.001\\
FNO++ & 4.15M & 2 & 0.154 $\pm$ 0.005 & 0.173 $\pm$ 0.003 & 0.154 $\pm$ 0.006\\
FNO++ & 7.71M & 4 & 0.151 $\pm$ 0.003 & 0.165 $\pm$ 0.004 & 0.149 $\pm$ 0.003\\
\midrule
FNO-WT & 2.37M & 1 & 0.151 $\pm$ 0.007 & 0.173 $\pm$ 0.017 & \textbf{0.126 $\pm$ 0.027}\\
FNO-DEQ & 2.37M &  1 & \textbf{0.128 $\pm$ 0.004} & \textbf{0.144 $\pm$ 0.007} & 0.136 $\pm$ 0.003 \\
\bottomrule
\end{tabular}}
\caption{Results on incompressible steady-state Navier-Stokes (viscosity=0.001): clean data (Col 4), noisy inputs (Col 5) and noisy observations (Col 6) with max variance of added noise being $(\sigma^2_{\max})^i$ and $(\sigma^2_{\max})^t$, respectively. Reported test error has been averaged on three different runs with seeds 0, 1, and 2. 
} 

\label{table:results-navier-stokes-visc-0.001-nl-0.001}
\end{table*}
\vspace{-4mm}

\begin{table*}[th!]
\centering
\resizebox{0.95\textwidth}{!}{%
\begin{tabular}{cccccc}
\toprule
\multirow{3}{*}{Architecture} & \multirow{3}{*}{Parameters} & \multirow{3}{*}{\#Blocks} & \multicolumn{3}{c}{Test error $\downarrow$} \\
\cmidrule(lr){4-6}
& & & $\sigma^2_{\max}=0$ & $(\sigma^2_{\max})^i=0.001$ & $(\sigma^2_{\max})^t=0.001$ \\
\midrule
FNO & 2.37M & 1 & 0.181 $\pm$ 0.005 & 0.186 $\pm$ 0.003 & 0.178 $\pm$ 0.006 \\
FNO & 4.15M & 2 & 0.138 $\pm$ 0.007 & 0.150 $\pm$ 0.006 & 0.137 $\pm$ 0.012\\
FNO & 7.71M & 4 & 0.152 $\pm$ 0.006 & 0.163 $\pm$ 0.002 & 0.151 $\pm$ 0.008\\
\midrule
FNO++ & 2.37M & 1 & 0.188 $\pm$ 0.002 & 0.207 $\pm$ 0.004& 0.187 $\pm$ 0.003 \\
FNO++ & 4.15M & 2 & 0.139 $\pm$ 0.004 & 0.153 $\pm$ 0.002 & 0.140 $\pm$ 0.005 \\
FNO++ & 7.71M & 4 & 0.130 $\pm$ 0.005 & 0.151 $\pm$ 0.004 & 0.128 $\pm$ 0.009 \\
\midrule
FNO-WT & 2.37M & 1 & 0.099 $\pm$ 0.007 & 0.101 $\pm$ 0.007 & 0.130 $\pm$ 0.044 \\
FNO-DEQ & 2.37M & 1 & \textbf{0.088 $\pm$ 0.006} & \textbf{0.099 $\pm$ 0.007} & \textbf{0.116 $\pm$ 0.011} \\
\bottomrule
\end{tabular}}

\caption{Results on incompressible steady-state Navier-Stokes (viscosity=0.01): clean data (Col 4), noisy inputs (Col 5) and noisy observations (Col 6) with max variance of added noise being $(\sigma^2_{\max})^i$ and $(\sigma^2_{\max})^t$, respectively. Reported test error has been averaged on three different runs with seeds 0, 1, and 2.} 
\label{table:results-navier-stokes-visc-0.01-nl-0.001}
\end{table*}

\section{Universal Approximation and Fast Convergence of FNO-DEQ}
\label{subsec:universal_approximation}

Though the primary contribution of our paper is empirical, we show (by fairly standard techniques) that FNO-DEQ is a universal approximator, under mild conditions on the steady-state PDEs. Moreover, we also show that in some cases, we can hope the fixed-point solver can converge rapidly.  

As noted in Definition~\ref{def:steady_state_PDE}, 
we have $\Omega:= \sT^d$.
We note that all continuous function 
$f \in L^2(\Omega; \R)$ and 
$\int_\Omega |f(x)| dx < \infty$
can be written as,
$f(x) = \sum_{\omega \in \Z^d} e^{i x^T \omega} \hat{f}_w$. 
where $\{\hat{f}_\omega\}_{\omega \in \Z^d}$
are the Fourier coefficients of the function $f$.
We define as $\lln$ as the space of functions such that for all $f_N \in \lln$ 
with Fourier coefficients that vanish outside a bounded ball.
Finally, we define an orthogonal projection operator $\project: \ll \to \lln$, 
such that 
for all $f \in \ll$ we have,
\begin{equation}
    \label{eq:projection_definition}
    f_n = \project(f) = \project\left(\sum_{\omega \in \Z^d} f_\omega e^{i x^T \omega}\right) = 
    \sum_{\|\omega\|_\infty \leq N} \hat{f}_\omega e^{i x^T \omega}.
\end{equation}
That is, the projection operator $\project$ takes an infinite dimensional function and projects it 
to a finite dimensional space. We prove the following universal approximation result:
\begin{theorem}
    \label{thm:main_theorem}
    Let $u^\star \in \lldu$ define the solution to a
    steady-state PDE
    in 
    Definition~\ref{def:steady_state_PDE},
    Then there exists an operator 
    $\gG: \lldu \times \lldf \to \lldu$
    such that,
    $u^\star = \gG(u^\star, f)$.
    Furthermore,
    for every $\epsilon > 0$ there exists an $N \in \N$
    such that
    for compact sets $K_u \subset \lldu$ and $K_f \subset \lldf$
    there exists a neural network 
    $G_\theta: \llndu \times \llndf \to \llndu$
    with parameters $\theta$,
    such that,
    $$\sup_{u \in K_u, f \in K_f} \|u^\star- G_\theta(\project u^\star, \project f)\|_{\ll} \leq \epsilon.$$
\end{theorem}
The proof for the above theorem is relatively straightforward and provided
in Appendix~\ref{sec:proof_of_universal_approximation}.
The proof uses the fact
that $u^\star$ 
is a fixed-point of the operator
$G(u, f) = u - (L(u) - f)$, allowing us to use the 
the results in~\citet{kovachki2021universal}
that show a continuous operator can be approximated
by a network as defined in \eqref{eq:fno_layer_def}.
Note that the choice of $G$ is by no means unique: one can ``universally approximate'' any operator $G(u, f) = u - A(L(u) - f)$, for a continuous operator $A$. 
Such a $G$ can be thought of as a form of ``preconditioned'' gradient descent, for a preconditioner $A$. For example, a Newton update has the form
$G(u,f) = u - L'(u)^{-1} \left(L(u) - f\right)$
,
where $L': L^2(\Omega; \R^{d_u}) \to L^2(\Omega; \R^{d_u})$
is the Frechet derivative of the operator $L$.  

The reason this is relevant is that the DEQ can choose to universally approximate a fixed-point equation for which the fixed-point solver it is trained with also converges rapidly. As an example, the following classical result shows that under Lax-Milgram-like conditions (a kind of strong convexity condition), Newton's method %
converges doubly exponentially fast:
\begin{lemma}[\cite{farago2002numerical}, Chapter 5]
    \label{lemma:fast_convergence}
   Consider the PDE defined Definition~\ref{def:steady_state_PDE},
   such that $d_u=d_v=d_f=1$.
   such that $L'(u)$ defines the Frechet derivative 
   of the operator $L$.
   If for all $u,v \in L^2(\Omega; \R)$ we have
   $\| L'(u) v\|_{\ll} \geq \lambda \|v\|_{\ll}$
   and 
   $\|L'(u) - L'(v)\|_{\ll} \leq \Lambda \|u - v\|_{\ll}$
   for $0 < \lambda \leq \Lambda <\infty $,
   then for the Newton update,
   $
       u_{t+1} \leftarrow u_t - L'(u_t)^{-1}\left(L(u_t) - f\right),
   $
   with $u_0 \in L^2(\Omega; \R)$, there exists an $\epsilon > 0$,
   such that  $\|u_T - u^\star\|_{\ll} \leq \epsilon$
   if 
    $
       T \geq \log 
       \left(
           \log \left(\frac{1}{\epsilon}\right) 
           /
           \log \left(\frac{2\lambda^2}{\Lambda\|L(u_0) - f\|_{\ll}}\right)
        \right).
    $
\end{lemma}

For completeness, we include the proof of the above lemma in the Appendix (Section ~\ref{sec:fast_convergence}).
We note that the conditions of the above lemma are satisfied for elliptic PDEs like Darcy Flow,
as well as many variational non-linear elliptic PDEs (e.g., those considered in ~\citet{marwah2022neural}). 
Hence, we can expect FNO-DEQs to quickly converge to the fixed point, since they employ quasi-Newton methods
like Broyden and Anderson methods~\citep{broyden1965class, anderson1965iterative}. %

\section{Conclusion}

In this work, we demonstrate that the inductive bias of deep equilibrium models---and weight-tied networks in general---makes them ideal architectures for approximating neural operators for steady-state PDEs.
Our experiments on steady-state Navier-Stokes equation 
and Darcy flow equations show that weight-tied models and FNO-DEQ 
perform outperform FNO models with $\sim 4\times$ the number of 
parameters and depth. 
Our findings indicate that FNO-DEQ and weight-tied architectures 
are, in general, more robust to both input and observation noise compared to non-weight-tied architectures, including FNO.
We believe that our results complement any future progress in the design and development of PDE solvers \citep{tran2021factorized, li2022fourier} for steady-state PDEs,
and hope that our work motivates the study of relevant inductive biases that could be used to improve them.

\section{Acknowledgements}
TM 
is supported 
by CMU Software Engineering Institute via Department of Defense under contract FA8702-15-D-0002.
AP
is supported 
by a grant from the Bosch Center for Artificial Intelligence.
ZL gratefully acknowledges the NSF (FAI 2040929 and IIS2211955), UPMC, Highmark Health, Abridge, Ford Research, Mozilla, the PwC Center, Amazon AI, JP Morgan Chase, the Block Center, the Center for Machine Learning and Health, and the CMU Software Engineering Institute (SEI) via Department of Defense contract FA8702-15-D-0002, for their generous support of ACMI Lab’s research.
JL
is supported in part by NSF award DMS-2012286, and 
AR
is
supported in part by NSF awards IIS-2211907, CCF-2238523, Amazon Research Award, and the  CMU/PwC DT\&I Center.

\bibliography{references}

\begin{thebibliography}{41}
\providecommand{\natexlab}[1]{#1}
\providecommand{\url}[1]{\texttt{#1}}
\expandafter\ifx\csname urlstyle\endcsname\relax
  \providecommand{\doi}[1]{doi: #1}\else
  \providecommand{\doi}{doi: \begingroup \urlstyle{rm}\Url}\fi

\bibitem[Anderson(1965)]{anderson1965iterative}
Donald~G Anderson.
\newblock Iterative procedures for nonlinear integral equations.
\newblock \emph{Journal of the ACM (JACM)}, 1965.

\bibitem[Anil et~al.(2022)Anil, Pokle, Liang, Treutlein, Wu, Bai, Kolter, and Grosse]{anil2022path}
Cem Anil, Ashwini Pokle, Kaiqu Liang, Johannes Treutlein, Yuhuai Wu, Shaojie Bai, J~Zico Kolter, and Roger~B Grosse.
\newblock Path independent equilibrium models can better exploit test-time computation.
\newblock \emph{Advances in Neural Information Processing Systems}, 35:\penalty0 7796--7809, 2022.

\bibitem[Bai et~al.(2019)Bai, Kolter, and Koltun]{bai2019deep}
Shaojie Bai, J~Zico Kolter, and Vladlen Koltun.
\newblock Deep equilibrium models.
\newblock \emph{Advances in Neural Information Processing Systems}, 32, 2019.

\bibitem[Bar-Sinai et~al.(2019)Bar-Sinai, Hoyer, Hickey, and Brenner]{bar2019learning}
Yohai Bar-Sinai, Stephan Hoyer, Jason Hickey, and Michael~P Brenner.
\newblock Learning data-driven discretizations for partial differential equations.
\newblock \emph{Proceedings of the National Academy of Sciences}, 116\penalty0 (31):\penalty0 15344--15349, 2019.

\bibitem[Batchelor and Batchelor(1967)]{batchelor1967introduction}
Cx~K Batchelor and George~Keith Batchelor.
\newblock \emph{An introduction to fluid dynamics}.
\newblock Cambridge university press, 1967.

\bibitem[Bhattacharya et~al.(2021)Bhattacharya, Hosseini, Kovachki, and Stuart]{bhattacharya2021model}
Kaushik Bhattacharya, Bamdad Hosseini, Nikola~B Kovachki, and Andrew~M Stuart.
\newblock Model reduction and neural networks for parametric {PDEs}.
\newblock \emph{The SMAI journal of computational mathematics}, 7:\penalty0 121--157, 2021.

\bibitem[Boyd(2001)]{boyd2001chebyshev}
John~P Boyd.
\newblock \emph{Chebyshev and {F}ourier spectral methods}.
\newblock Courier Corporation, 2001.

\bibitem[Brandstetter et~al.(2022)Brandstetter, Worrall, and Welling]{brandstetter2022message}
Johannes Brandstetter, Daniel Worrall, and Max Welling.
\newblock Message passing neural pde solvers.
\newblock \emph{arXiv preprint arXiv:2202.03376}, 2022.

\bibitem[Broyden(1965)]{broyden1965class}
Charles~G Broyden.
\newblock A class of methods for solving nonlinear simultaneous equations.
\newblock \emph{Mathematics of computation}, 1965.

\bibitem[Chen and Chen(1995)]{chen1995universal}
Tianping Chen and Hong Chen.
\newblock Universal approximation to nonlinear operators by neural networks with arbitrary activation functions and its application to dynamical systems.
\newblock \emph{IEEE Transactions on Neural Networks}, 6\penalty0 (4):\penalty0 911--917, 1995.

\bibitem[Chen et~al.(2021)Chen, Lu, and Lu]{chen2021representation}
Ziang Chen, Jianfeng Lu, and Yulong Lu.
\newblock On the representation of solutions to elliptic pdes in barron spaces.
\newblock \emph{Advances in neural information processing systems}, 34:\penalty0 6454--6465, 2021.

\bibitem[Dresdner et~al.(2022)Dresdner, Kochkov, Norgaard, Zepeda-Núñez, Smith, Brenner, and Hoyer]{Dresdner2022-Spectral-ML}
Gideon Dresdner, Dmitrii Kochkov, Peter Norgaard, Leonardo Zepeda-Núñez, Jamie~A. Smith, Michael~P. Brenner, and Stephan Hoyer.
\newblock Learning to correct spectral methods for simulating turbulent flows.
\newblock 2022.
\newblock \doi{10.48550/ARXIV.2207.00556}.
\newblock URL \url{https://arxiv.org/abs/2207.00556}.

\bibitem[Farag{\'o} and Kar{\'a}tson(2002)]{farago2002numerical}
Istv{\'a}n Farag{\'o} and J{\'a}nos Kar{\'a}tson.
\newblock \emph{Numerical solution of nonlinear elliptic problems via preconditioning operators: Theory and applications}, volume~11.
\newblock Nova Publishers, 2002.

\bibitem[Fung et~al.(2022)Fung, Heaton, Li, McKenzie, Osher, and Yin]{fung2022jfb}
Samy~Wu Fung, Howard Heaton, Qiuwei Li, Daniel McKenzie, Stanley Osher, and Wotao Yin.
\newblock Jfb: Jacobian-free backpropagation for implicit networks.
\newblock In \emph{Proceedings of the AAAI Conference on Artificial Intelligence}, volume~36, pages 6648--6656, 2022.

\bibitem[Geng et~al.(2021)Geng, Zhang, Bai, Wang, and Lin]{geng2021training}
Zhengyang Geng, Xin-Yu Zhang, Shaojie Bai, Yisen Wang, and Zhouchen Lin.
\newblock On training implicit models.
\newblock \emph{Advances in Neural Information Processing Systems}, 34:\penalty0 24247--24260, 2021.

\bibitem[Hao et~al.(2023)Hao, Wang, Su, Ying, Dong, Liu, Cheng, Song, and Zhu]{hao2023gnot}
Zhongkai Hao, Zhengyi Wang, Hang Su, Chengyang Ying, Yinpeng Dong, Songming Liu, Ze~Cheng, Jian Song, and Jun Zhu.
\newblock Gnot: A general neural operator transformer for operator learning.
\newblock In \emph{International Conference on Machine Learning}, pages 12556--12569. PMLR, 2023.

\bibitem[Hendrycks and Gimpel(2016)]{gelu}
Dan Hendrycks and Kevin Gimpel.
\newblock Gaussian error linear units (gelus).
\newblock \emph{arXiv preprint arXiv:1606.08415}, 2016.

\bibitem[Hornik et~al.(1989)Hornik, Stinchcombe, and White]{hornik1989multilayer}
Kurt Hornik, Maxwell Stinchcombe, and Halbert White.
\newblock Multilayer feedforward networks are universal approximators.
\newblock \emph{Neural networks}, 2\penalty0 (5):\penalty0 359--366, 1989.

\bibitem[Hsieh et~al.(2019)Hsieh, Zhao, Eismann, Mirabella, and Ermon]{hsieh2019learning}
Jun-Ting Hsieh, Shengjia Zhao, Stephan Eismann, Lucia Mirabella, and Stefano Ermon.
\newblock Learning neural pde solvers with convergence guarantees.
\newblock \emph{arXiv preprint arXiv:1906.01200}, 2019.

\bibitem[Kochkov et~al.(2021)Kochkov, Smith, Alieva, Wang, Brenner, and Hoyer]{kochkov2021machine}
Dmitrii Kochkov, Jamie~A Smith, Ayya Alieva, Qing Wang, Michael~P Brenner, and Stephan Hoyer.
\newblock Machine learning--accelerated computational fluid dynamics.
\newblock \emph{Proceedings of the National Academy of Sciences}, 118\penalty0 (21):\penalty0 e2101784118, 2021.

\bibitem[Kopriva(2009)]{kopriva2009implementing}
David~A Kopriva.
\newblock \emph{Implementing spectral methods for partial differential equations: Algorithms for scientists and engineers}.
\newblock Springer Science \& Business Media, 2009.

\bibitem[Kovachki et~al.(2021{\natexlab{a}})Kovachki, Lanthaler, and Mishra]{kovachki2021universal}
Nikola Kovachki, Samuel Lanthaler, and Siddhartha Mishra.
\newblock On universal approximation and error bounds for {F}ourier neural operators.
\newblock \emph{The Journal of Machine Learning Research}, 22\penalty0 (1):\penalty0 13237--13312, 2021{\natexlab{a}}.

\bibitem[Kovachki et~al.(2021{\natexlab{b}})Kovachki, Li, Liu, Azizzadenesheli, Bhattacharya, Stuart, and Anandkumar]{kovachki2021neural}
Nikola Kovachki, Zongyi Li, Burigede Liu, Kamyar Azizzadenesheli, Kaushik Bhattacharya, Andrew Stuart, and Anima Anandkumar.
\newblock Neural operator: Learning maps between function spaces.
\newblock \emph{arXiv preprint arXiv:2108.08481}, 2021{\natexlab{b}}.

\bibitem[Kovachki et~al.(2023)Kovachki, Li, Liu, Azizzadenesheli, Bhattacharya, Stuart, and Anandkumar]{kovachki2023neural}
Nikola Kovachki, Zongyi Li, Burigede Liu, Kamyar Azizzadenesheli, Kaushik Bhattacharya, Andrew Stuart, and Anima Anandkumar.
\newblock Neural operator: Learning maps between function spaces with applications to {PDEs}.
\newblock \emph{Journal of Machine Learning Research}, 24\penalty0 (89):\penalty0 1--97, 2023.

\bibitem[Lanthaler et~al.(2022)Lanthaler, Mishra, and Karniadakis]{lanthaler2022error}
Samuel Lanthaler, Siddhartha Mishra, and George~E Karniadakis.
\newblock Error estimates for deeponets: A deep learning framework in infinite dimensions.
\newblock \emph{Transactions of Mathematics and Its Applications}, 6\penalty0 (1):\penalty0 tnac001, 2022.

\bibitem[LeVeque(2007)]{leveque2007finite}
Randall~J LeVeque.
\newblock \emph{Finite difference methods for ordinary and partial differential equations: steady-state and time-dependent problems}.
\newblock SIAM, 2007.

\bibitem[Li et~al.(2022{\natexlab{a}})Li, Meidani, and Farimani]{li2022transformer}
Zijie Li, Kazem Meidani, and Amir~Barati Farimani.
\newblock Transformer for partial differential equations' operator learning.
\newblock \emph{arXiv preprint arXiv:2205.13671}, 2022{\natexlab{a}}.

\bibitem[Li et~al.(2020{\natexlab{a}})Li, Kovachki, Azizzadenesheli, Liu, Bhattacharya, Stuart, and Anandkumar]{li2020fourier}
Zongyi Li, Nikola Kovachki, Kamyar Azizzadenesheli, Burigede Liu, Kaushik Bhattacharya, Andrew Stuart, and Anima Anandkumar.
\newblock Fourier neural operator for parametric partial differential equations.
\newblock \emph{arXiv preprint arXiv:2010.08895}, 2020{\natexlab{a}}.

\bibitem[Li et~al.(2020{\natexlab{b}})Li, Kovachki, Azizzadenesheli, Liu, Bhattacharya, Stuart, and Anandkumar]{li2020neural}
Zongyi Li, Nikola Kovachki, Kamyar Azizzadenesheli, Burigede Liu, Kaushik Bhattacharya, Andrew Stuart, and Anima Anandkumar.
\newblock Neural operator: Graph kernel network for partial differential equations.
\newblock \emph{arXiv preprint arXiv:2003.03485}, 2020{\natexlab{b}}.

\bibitem[Li et~al.(2022{\natexlab{b}})Li, Huang, Liu, and Anandkumar]{li2022fourier}
Zongyi Li, Daniel~Zhengyu Huang, Burigede Liu, and Anima Anandkumar.
\newblock Fourier neural operator with learned deformations for pdes on general geometries.
\newblock \emph{arXiv preprint arXiv:2207.05209}, 2022{\natexlab{b}}.

\bibitem[Liao et~al.(2018)Liao, Xiong, Fetaya, Zhang, Yoon, Pitkow, Urtasun, and Zemel]{liao2018reviving}
Renjie Liao, Yuwen Xiong, Ethan Fetaya, Lisa Zhang, KiJung Yoon, Xaq Pitkow, Raquel Urtasun, and Richard Zemel.
\newblock Reviving and improving recurrent back-propagation.
\newblock In \emph{International Conference on Machine Learning}, pages 3082--3091. PMLR, 2018.

\bibitem[Liu et~al.(2022)Liu, Xu, and Zhang]{Liu2022-nd}
Xinliang Liu, Bo~Xu, and Lei Zhang.
\newblock {Mitigating spectral bias for the multiscale operator learning with hierarchical attention}.
\newblock October 2022.
\newblock URL \url{http://arxiv.org/abs/2210.10890}.

\bibitem[Lu et~al.(2019)Lu, Jin, and Karniadakis]{lu2019deeponet}
Lu~Lu, Pengzhan Jin, and George~Em Karniadakis.
\newblock Deeponet: Learning nonlinear operators for identifying differential equations based on the universal approximation theorem of operators.
\newblock \emph{arXiv preprint arXiv:1910.03193}, 2019.

\bibitem[Marwah et~al.(2021)Marwah, Lipton, and Risteski]{marwah2021parametric}
Tanya Marwah, Zachary Lipton, and Andrej Risteski.
\newblock Parametric complexity bounds for approximating {PDEs} with neural networks.
\newblock \emph{Advances in Neural Information Processing Systems}, 34:\penalty0 15044--15055, 2021.

\bibitem[Marwah et~al.(2022)Marwah, Lipton, Lu, and Risteski]{marwah2022neural}
Tanya Marwah, Zachary~C Lipton, Jianfeng Lu, and Andrej Risteski.
\newblock Neural network approximations of {PDEs} beyond linearity: Representational perspective.
\newblock \emph{arXiv preprint arXiv:2210.12101}, 2022.

\bibitem[Moukalled et~al.(2016)Moukalled, Mangani, Darwish, Moukalled, Mangani, and Darwish]{moukalled2016finite}
Fadl Moukalled, Luca Mangani, Marwan Darwish, F~Moukalled, L~Mangani, and M~Darwish.
\newblock \emph{The finite volume method}.
\newblock Springer, 2016.

\bibitem[Patankar and Spalding(1983)]{patankar1983calculation}
Suhas~V Patankar and D~Brian Spalding.
\newblock A calculation procedure for heat, mass and momentum transfer in three-dimensional parabolic flows.
\newblock In \emph{Numerical prediction of flow, heat transfer, turbulence and combustion}, pages 54--73. Elsevier, 1983.

\bibitem[Patel et~al.(2021)Patel, Trask, Wood, and Cyr]{patel2021physics}
Ravi~G Patel, Nathaniel~A Trask, Mitchell~A Wood, and Eric~C Cyr.
\newblock A physics-informed operator regression framework for extracting data-driven continuum models.
\newblock \emph{Computer Methods in Applied Mechanics and Engineering}, 373:\penalty0 113500, 2021.

\bibitem[Revay et~al.(2020)Revay, Wang, and Manchester]{revay2020lipschitz}
Max Revay, Ruigang Wang, and Ian~R Manchester.
\newblock Lipschitz bounded equilibrium networks.
\newblock \emph{arXiv preprint arXiv:2010.01732}, 2020.

\bibitem[Tran et~al.(2021)Tran, Mathews, Xie, and Ong]{tran2021factorized}
Alasdair Tran, Alexander Mathews, Lexing Xie, and Cheng~Soon Ong.
\newblock Factorized {F}ourier neural operators.
\newblock \emph{arXiv preprint arXiv:2111.13802}, 2021.

\bibitem[Winston and Kolter(2020)]{winston2020monotone}
Ezra Winston and J~Zico Kolter.
\newblock Monotone operator equilibrium networks.
\newblock \emph{Advances in neural information processing systems}, 33:\penalty0 10718--10728, 2020.

\end{thebibliography}
\bibliographystyle{plainnat}

\clearpage
\appendix
\section*{Appendix}
\section{Implementation Details}
\label{sec:implementation_details}

\paragraph{Training details.} We train all the networks for 500 epochs  with Adam optimizer. 
The learning rate is set to 0.001 for Darcy flow and 0.005 for Navier-Stokes. We use learning rate weight decay of 1e-4 for both Navier-Stokes and Darcy flow. The batch size is set to 32.
In case of Darcy flow, we also use cosine annealing for learning rate scheduling. 
We run all our experiments on a combination of NVIDIA RTX A6000, NVIDIA GeForce RTX 2080 Ti and 3080 Ti. 
All networks can easily fit on a single NVIDIA RTX A6000, but training time varies between the networks. 

For FNO-DEQ, we use Anderson solver~\citep{anderson1965iterative} to solve for the fixed point in the forward pass. The maximum number of Anderson solver steps is kept fixed at 32 for Dary Flow, and 16 for Navier Stokes. For the backward pass, we use phantom gradients~\citep{geng2021training} which are computed as:
\begin{equation}
    \label{eq:phatom-grad}
    u^\star = \tau G_\theta(u^\star, a) + (1 - \tau)u^\star
\end{equation}
where $\tau$ is a tunable damping factor and $u^\star$ is the fixed point computed using Anderson solver in the forward pass. This step can be repeated $S$ times. We use $\tau=0.5$ and $S=1$ for Darcy Flow, and $\tau=0.8$ and $S=3$ for Navier-Stokes. 

For the S-FNO-DEQ used in \cref{table:results-darcy-flow-all}, we use Broyden's method \citep{broyden1965class} to solve for the fixed point in the forward pass and use exact implicit gradients, computed through implicit function theorem as shown in \cref{eq:implcit-grad-deq}, for the backward pass through DEQ. The maximum number of solver steps is fixed at 32.

For weight-tied networks, we repeatedly apply the FNO block to the input $12$ times for Darcy flow, and $6$ times for Navier-Stokes.

\paragraph{Network architecture details.} The width of an FNO layer set to 32 across all the networks. Additionally, we retain only 12 Fourier modes in FNO layer, and truncate higher Fourier modes. We use the code provided by \citet{li2020fourier} to replicate the results for FNO, and construct rest of the networks on top of this as described in \cref{sec:experiments}.

\section{Datasets}

\subsection{Darcy Flow}
\label{subsec:Darcy_flow_implementation}
As mentioned in \cref{sec:experiments} 
we use the dataset provided by~\cite{li2020fourier} for our experiments with steady-state Darcy-Flow.

All the models are trained on 1024 data samples and tested on 500 samples. The resolution of original images is $421 \times 421$ which we downsample to $85 \times 85$ for our experiments. For experiments with noisy inputs/observations, the variance of Gaussian noise that we add to PDEs are [0, 1e-9, 1e-8, 1e-7, 1e-6, 1e-5, 1e-4, 1e-3]. 
\begin{figure}[!htbp]
    \centering
    \includegraphics[height=\textwidth]{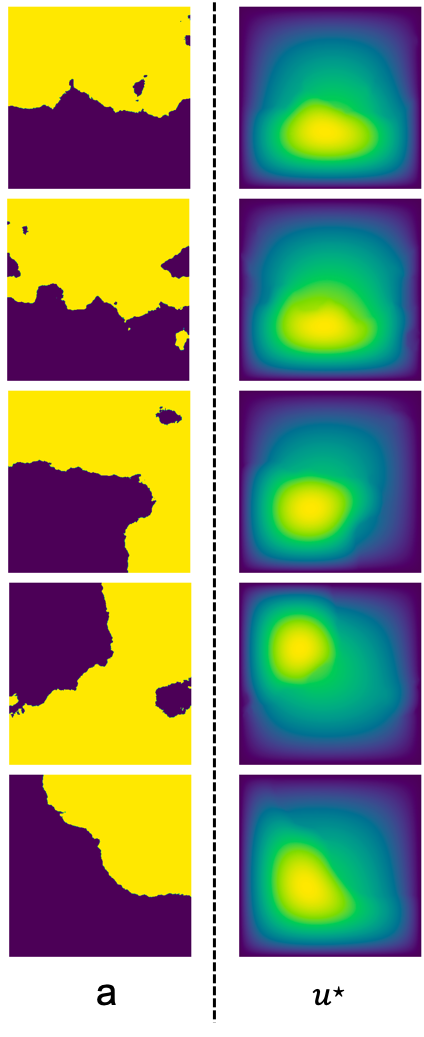}
    \label{fig:darcy_flow}
    \caption{Samples from Darcy Flow}
\end{figure}

\subsection{Steady-State Incompressible Fluid Navier-Stoke}
\label{subsec:navier_stokes_implementation}
\begin{align*}
    u \cdot \nabla \omega &= \nu \Delta \omega + f, \qquad x \in \Omega\\
    \nabla \cdot u &= 0 \qquad\qquad\quad\;\;  x \in \Omega
\end{align*}
To generate the dataset for steady-state Navier-Stokes, 
instead of solving the steady state PDE using 
steady-state solvers like the SIMPLE algorithm~\cite{patankar1983calculation},
we first choose the solution $\omega^\star := \nabla \times u^\star$
of the PDE and then generate the corresponding equation, i.e. calculate the corresponding force term 
$f = u^\star \cdot \nabla \omega^\star - \nu \Delta \omega^\star.$

To generate the solutions $\omega^\star$, we forward propagate a relatively simple initial distribution of $\omega_0$ (sampled from a Gaussian random field) through a
time-dependent Navier-Stokes equation in the vorticity form for a short period of time. This ensures our dataset contains solutions $\omega^*$ that are rich and complex.  Precisely, recall the Navier-Stokes equations in their vorticity form:
\begin{equation}
\label{eq:navier_stokes_time_dependent}
\begin{split}
    \partial_t \omega(x,t) + u(x,t) \cdot \nabla \omega(x,t)
    &= \nu \Delta \omega(x,t) + g(x)  \qquad x \in (0, 2\pi)^2, t \in [0, T]\\
    \nabla \cdot u(x,t) &= 0 \qquad x \in (0, 2\pi)^2, t \in [0, T]\\
    \omega(x, 0) &= \omega_0(x) \qquad x \in (0, 2\pi)^2
\end{split}
\end{equation}
where $g(x) = \nabla \times \tilde{g}(x)$ and 
$\tilde{g}(x) = \sin(5x_1)\hat{x_2}$ 
is a divergence free forcing term and $x = (x_1, x_2)$ are the two coordinates of the input vector.
We forward propagate the equations \eqref{eq:navier_stokes_time_dependent}
using a pseudo-spectral method
using the functions provided in JAX-CFD~\citep{kochkov2021machine,Dresdner2022-Spectral-ML} package. 
The initial vorticity $\omega_0$ is sampled from a 
Gaussian random field $\gN(0, (5^{3/2}(I + 25\Delta)^{-2.5}))$, which is then made divergence free.
We forward propagate the Navier-Stokes equation in~\eqref{eq:navier_stokes_time_dependent}
for time $T = 0.5$ with $dt=0.002$ to get $\omega(1, x)$, 
which we choose as the solution to the steady-state PDE in~\eqref{eq:navier_stokes}, i.e, $\omega^\star$
for Equation~\ref{eq:navier_stokes}.

Subsequently, we use the stream function 
$\Psi$~\citep{batchelor1967introduction} to calculate 
$u = \left(\partial \Psi/\partial x_1, \partial \Psi/\partial x_2\right)$ 
by solving the Poisson equation $\Delta \Psi = \omega$ 
in the Fourier domain.
Furthermore, since 
$f = u^\star \cdot \nabla \omega^\star - \nu \Delta \omega^\star$
we use the stream function to calculate $(f_1, f_2)$, i.e., the different components of the force term.

We use $4500$ training samples and $500$ testing samples.
The input to the network is the vector field $\tilde{f} = (f_1, f_2)$ 
and we learn a map that outputs the vorticity $\omega^\star$.
The resolution of grid used to generate the dataset is $256 \times 256$ 
which we downsample to $128 \times 128$ while training the models. For experiments with noisy inputs/observations, we consider two values of maximum  variance of Gaussian noise: 1e-3 and 4e-3. 
The variances of the Gaussian noise that we add to the PDEs for the latter case are [0, 1e-9, 1e-8, 1e-7, 1e-6, 1e-5, 1e-4, 1e-3, 2e-3, 4e-3].
However, when conducting experiments with a variance of 1e-3, we exclude the last two values of variance from this list.

\begin{figure}[!htbp]
    \centering
    \includegraphics[height=\textwidth]{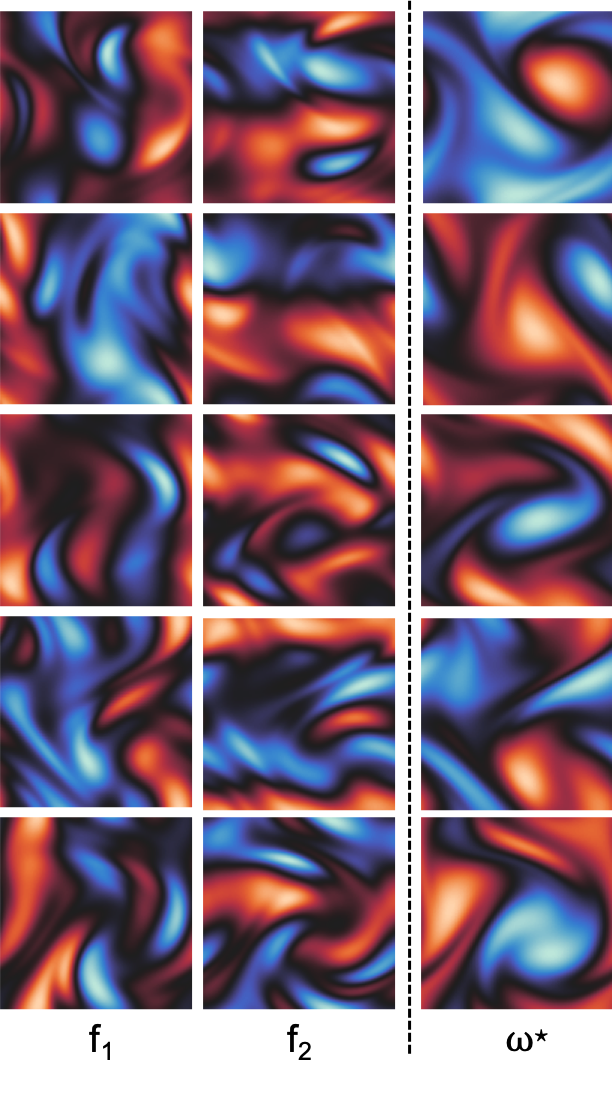}
    \label{fig:navier-stokes-data-visc-0.001}
    \caption{Samples from Steady-state Navier-Stokes dataset with viscosity $0.001$. 
    Each triplet visualizes the inputs $f_1$, $f_2$ and the ground truth output i.e. $\omega^\star$. }
\end{figure}

\begin{figure}[!htbp]
    \centering
    \includegraphics[height=\textwidth]{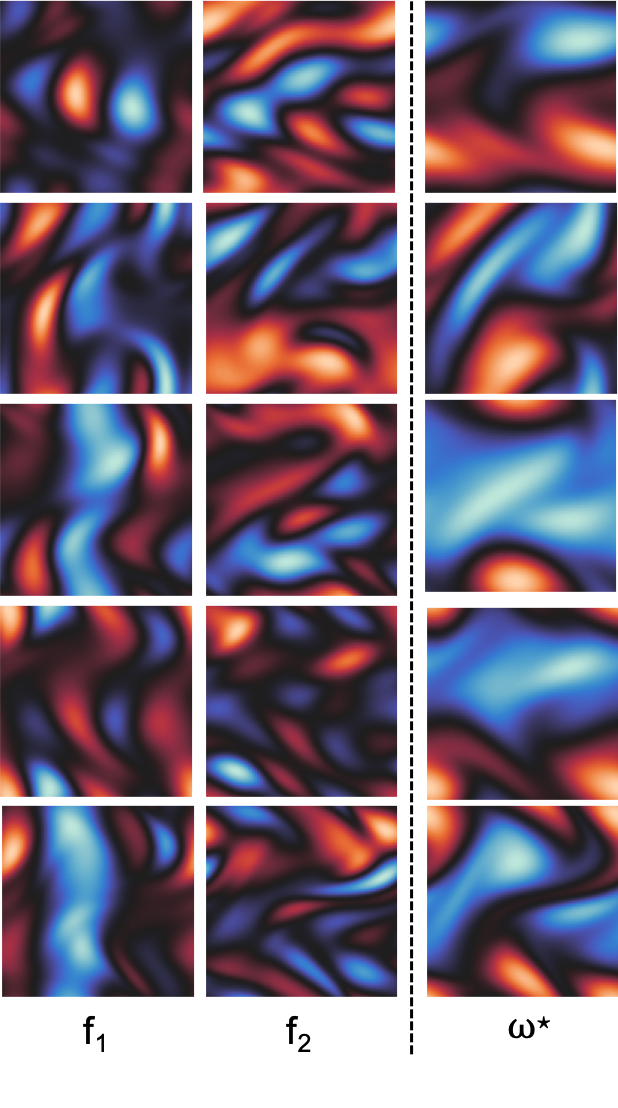}
    \label{fig:navier-stokes-data-visc-0.01}
    \caption{Samples from Steady-state Navier-Stokes dataset with viscosity $0.01$. 
    Each triplet visualizes the inputs $f_1$, $f_2$ and the ground truth output i.e. $\omega^\star$. }
\end{figure}
\section{Proof of Universal Approximation}
\label{sec:proof_of_universal_approximation}

The proof of the universal approximation essentially follows from the result on the universal approximation 
capabilities of FNO layers in~\citet{kovachki2021universal}, applied to $\calG(v,f) = v - (Lv - f)$. For the sake of completeness, we reitarate the key steps.  

For simplicity, we will assume that 
$d_u = d_v = d_f = 1$. (The results straightforwardly generalize.)
We will first establish some key technical lemmas and introduce some notation and definitions useful 
for the proof for Theorem~\ref{thm:main_theorem}.

\begin{definition}
    An operator $T: L^2(\Omega; \R) \to L^2(\Omega; \R)$
    is continuous 
    at $u \in L^2(\Omega; \R)$ if for every 
    $\epsilon > 0$, there exists a $\delta > 0$,
    such that for all $v \in \ll$ with $\|u - v\|_{\ll}\leq\delta$,
    we have 
    $\|L(u) - L(v)\|_{\ll} \leq \epsilon$.
\end{definition}

First, we approximate the infinite dimensional operator $\calG: \ll \times \ll \to \ll$
by projecting the functions in $\ll$ to a finite-dimensional approximation $\lln$, and considering the action of the operator on this subspace.  
The linear projection we use is the one introduced in~\eqref{eq:projection_definition}.
More precisely we show the following result, 

\begin{lemma}
    \label{lemma:discrete_close}
    Given a continuous operator $L: \ll \to \ll$ as defined in~\eqref{eq:lu_f}, let us define an operator $\calG: \ll \times \ll \to \ll$
 as $\calG(v, f):= v - (L(v) - f)$.
    Then, for every $\epsilon > 0$ there exists an $N \in \N$
    such that for all $v, f$ in any compact set $K \subset \ll$, 
    the operator $\calG_N = \project \calG(\project v, \project f)$ 
    is an $\epsilon$-approximation of $\calG(v, f)$, i.e., we have,
    $$\sup_{v,f \in K} \|\calG(v, f) - \calG_N(v, f)\|_{\ll} \leq \epsilon.$$
\end{lemma}
\begin{proof}
    Note that for an $\epsilon > 0$
    there exists an $N = N(\epsilon, d)$ 
    such that for all $v \in K$ 
    we have 
    $$\sup_{v \in K} \|v - \project v\|_{\ll} \leq \epsilon.$$
    Therefore, using the definition of $\calG_N$ we can bound the $\ll$ norm of the difference between 
    $\calG$ and $\calG_N$ as follows,
     \begin{align*}
         &\|\calG(v, f) - \project \calG(v_n, f_n)\|_{\ll} \\
         &\leq 
            \|\calG(v, f) - \project \calG(v, f)\|_{\ll} 
            + \|\project \calG(v, f) - \project \calG(\project v, \project f)\|_{\ll} \\
        &\leq 
            \underbrace{\|\calG(v, f) - \project \calG(v, f)\|_{\ll}}_{I}
            + \underbrace{\| \calG(v, f) - \calG(\project v, \project f)\|_{\ll}}_{II}
     \end{align*}
     We first bound the term $I$ as follows:
     \begin{align*}
        &\|\calG(v, f) - \project \calG(v, f)\|_{\ll} \\
            &= \left\|v - (L(v) - f) - \project\left(v - (L(v) - f)\right)\right\|_{\ll}\\
            &= \|v - \project v \|_{\ll} + \|f - \project f\|_{\ll} + \|L(v) - \project L(v)\|_{\ll} \\
            &= \epsilon + \epsilon + \|L(v) - \project L(v)\|_{\ll} \numberthis \label{eq:lemma_1_eq1}
     \end{align*}
     Since $L$ is continuous, 
     for all compact sets $K \subset \ll$, $L(K)$ is compact as well. This is because: (1) for any $u \in K$, $\|L(u)\|_{\ll}$ is finite; (2) for any $v \in K$, $\|L(v)\|_{\ll} \leq \|L(u)\|_{\ll} + C\|u-v\|_{\ll}$.  
     Therefore, 
     for every $\epsilon > 0$, there exists an $N\in\N$ such that 
     \begin{align*}
         \sup_{v \in K} \|L(v) -  \project L(v)\|_{\ll} \leq \epsilon.
     \end{align*}
     Substituting the above result in~\eqref{eq:lemma_1_eq1}, we have
     \begin{equation}
         \label{eq:lemma1_term1_upper_bound}
         \|\calG(v, f) - \project \calG(v, f)\|_{\ll} \leq 3\epsilon.
     \end{equation}

     Similarly, for all $v \in K$ where $K$ is compact, 
     we can bound Term $II$ as following,
     \begin{align*}
         &\left \|\calG(v, f) - \calG(\project v, \project f)\right\|_{\ll} \\
         &\leq \left\|v - (L(v) - f) - \project v - (L(\project v) - \project f)\right\|_{\ll} \\
         &\leq \|v - \project v\|_{\ll} + \|f - \project f\|_{\ll} + \|L(v) - L(\project v)\|_{\ll} \\
         &\leq \epsilon + \epsilon + \|L(v) - L(\project v)\|_{\ll}. \numberthis \label{eq:lemma_2_eq2}
     \end{align*}
    Now, since $v \in K$ and $L:\ll \to \ll$ is a continuous operator, 
    there exists a modulus of continuity (an increasing real valued function)
    $\alpha \in [0, \infty)$, such that
    for all $v \in K$, we have
    \begin{align*}
        \|L(v) - L(\project v)\|_{\ll} \leq \alpha\left(\|v - \project v\|_{\ll}\right)
    \end{align*}
    Hence for every $\epsilon > 0$ 
    there exists %
    an $N \in \N$ such that, 
    $$\alpha(\|v - \project v\|_{\ll}) \leq \epsilon.$$
    Plugging these bounds in~\eqref{eq:lemma_2_eq2}, we get, 
    \begin{equation}
         \label{eq:lemma1_term2_upper_bound}
         \left \|\calG(v, f) - \calG(\project v, \project f)\right\|_{\ll} 
         \leq 3\epsilon.
    \end{equation}
    Therefore, combining~\eqref{eq:lemma1_term1_upper_bound} and~\eqref{eq:lemma1_term2_upper_bound}
    then for
    $\epsilon > 0$, there exists an 
    $N \in \N$, such that for all $v,f \in K$ 
    we have
    \begin{align}
        \sup_{v,f\in K}\left\|\calG(v, f) - \project \calG(v_n, f_n)\right\|_{\ll} \leq 6\epsilon.
    \end{align}
    Taking $\epsilon' = 6\epsilon$ proves the claim.
\end{proof}

\begin{proof}[Proof of Theorem~\ref{thm:main_theorem}]
    For Lemma~\ref{lemma:discrete_close}
    we know that there exists a finite dimensional 
    projection for the operator $\gG$, defined as $\gG_N(v, f)$ such that 
    for all $v, f \in \ll$ we have
    $$\|\gG(v, f) - \gG_N(v, f)\|_{\ll} \leq \epsilon.$$

    Now using the definition of $\gG_N(v, f)$ we have
    \begin{align*}
        \gG_N(v, f) &= \project\gG(\project v, \project f)\\
        &= \project v - \left(\project L(\project v) - \project f\right)
    \end{align*}
    From~\citet{kovachki2021universal}, Theorem 2.4 we know 
    that there exists an FNO network $G_{\theta^L}$ 
    of the form defined in 
    ~\eqref{eq:fno_layer_def} such that 
    for all $v \in K$, where $K$ is a compact set, 
    there exists an $\epsilon^L$
    we have
    \begin{equation}
    \sup_{v \in K} \|\project L(\project v) - G_{\theta^L}\|_{\ll} \leq 
        \epsilon^L
    \end{equation}
    Finally, 
    note that from Lemma~D.1 in~\cite{kovachki2021universal},
    we have
    that for any $v \in K$, 
    there exists an FNO layers
    $G_{\theta^f} \in \ll$ and $G_{\theta^v} \in \ll$ defined 
    in~\eqref{eq:fno_layer}
    such that
    \begin{equation}
    \sup_{v \in K} \|\project v - G_{\theta^v}\|_{\ll} \leq 
        \epsilon^v
    \end{equation}
    and 
    \begin{equation}
    \sup_{f \in K} \|\project f - G_{\theta^f}\|_{\ll} \leq 
        \epsilon^f
    \end{equation}
    for $\epsilon^v > 0$ and $\epsilon^f > 0$.

    Therefore there exists an $\tilde{\epsilon} > $ 
    such that there is an FNO network
    $G_\theta: \ll \times \ll \to \ll$
    where $\theta := \{\theta^L, \theta^v, \theta^f\}$
    such that
    \begin{equation}
        \label{eq:neural_network_existence}
        \sup_{v \in K, f \in \ll} 
            \|\gG_N(v, f) - G_\theta(v, f)\|_{\ll} \leq 
            \tilde{\epsilon}
    \end{equation}

    Now, since we know that $u^\star$ is the fixed point 
    of the operator $\gG$ we have from Lemma~\ref{lemma:discrete_close} and \eqref{eq:neural_network_existence},
    \begin{align*}
        \|\gG(u^\star, f) - G_\theta(u^\star, f)\|_{\ll} 
        &\leq \|u^\star - \gG_N(u^\star, f)\|_{\ll} 
            + \|\gG_N(u^\star, f) -  G_\theta(u^\star, f)\|_{\ll}\\
        &\leq \tilde{\epsilon} + \epsilon.
    \end{align*}
\end{proof}

\section{Fast Convergence for Newton Method}
\label{sec:fast_convergence}

\begin{definition}[Frechet Derivative in $\ll$]
    For a continuous operator $F: \ll \to \ll$, 
    the Frechet derivative at $u \in \ll$
    is a linear operator $F'(u): \ll \to \ll$ such that for all $v\in \ll$ we have
    \begin{align*}
        \lim_{\|v\|_{\ll} \to 0} \frac{\|F(u + v) - F(u) - F'(u)(v)\|_{\ll}}{\|v\|_{\ll}} = 0.
    \end{align*}
\end{definition}

\begin{lemma}
    \label{lemma:frechet_upperbound}
    Given the operator $L: \ll \to \ll$
    with Frechet derivative $L'$,
    such that for all $u, v \in \ll$, we have
    $\|L'(u)(v)\|_{\ll} \geq \lambda \|v\|_{\ll}$, then 
    $L'(u)^{-1}$ exists and we have,
    for all $v_1, v_2 \in \ll$:
    \begin{enumerate}
        \item $\|L'(u)^{-1}(v_1)\|_{\ll} \leq \frac{1}{\lambda}\|v_1\|_{\ll}$.
        \item $\|v_1 - v_2\|_{\ll} 
            \leq \frac{1}{\lambda}\|L(v_1) - L(v_2)\|_{\ll}$
    \end{enumerate}
\end{lemma}
\begin{proof}
    Note that for all $u, v' \in \ll$ we have,
    \begin{align*}
        \|L'(u)v'\|_{\ll} \geq \lambda \|v'\|_{\ll}
    \end{align*}
    Taking $v = L'(u)^{-1}(v')$, we have
    \begin{align*}
        \|L'(u)\left(L'(u)^{-1}(v)\right)\|_{\ll} &\geq \lambda \|L^{-1}(u)(v)\|_{\ll}\\
        \implies
        \frac{1}{\lambda}\|v\|_{\ll} &\geq  \|L^{-1}(u)(v)\|_{\ll}.
    \end{align*}

    For part $2$, note that 
    there exists a $c \in [0, 1]$ 
    such that
    \begin{align*}
        \|L(v_1) - L(v_2)\|_{\ll} \geq \inf_{c \in [0, 1]}\|L'(c v_1  + (1 - c)v_2)\|_{2}\|v_1 - v_2\|_{\ll}
        \geq \lambda \|v_1 - v_2\|_{\ll}.
    \end{align*}
\end{proof}

We now show the proof for Lemma~\ref{lemma:fast_convergence}.
The proof is standard and can be found in~\cite{farago2002numerical}, 
however we include the 
complete proof here for the sake of completeness.

We restate the Lemma here for the convenience of the reader.
\begin{lemma}[\cite{farago2002numerical}, Chapter 5]
    \label{lemma:fast_convergence}
   Consider the PDE defined Definition~\ref{def:steady_state_PDE},
   such that $d_u=d_v=d_f=1$.
   such that $L'(u)$ defines the Frechet derivative 
   of the operator $L$.
   If for all $u,v \in L^2(\Omega; \R)$ we have
   $\| L'(u) v\|_{\ll} \geq \lambda \|v\|_{\ll}$
   \footnote{We note that this condition is different from the 
   condition on the inner-product in the submitted version 
   of the paper, which had. 
   $\langle L'(u), v\rangle_{\ll} \geq \lambda \|v\|_{\ll}$.
   }
   and 
   $\|L'(u) - L'(v)\|_{\ll} \leq \Lambda \|u - v\|_{\ll}$
   for $0 < \lambda \leq \Lambda <\infty $,
   then for the Newton update,
   $
       u_{t+1} \leftarrow u_t - L'(u_t)^{-1}\left(L(u_t) - f\right),
   $
   with $u_0 \in L^2(\Omega; \R)$, there exists an $\epsilon > 0$,
   such that  $\|u_T - u^\star\|_{\ll} \leq \epsilon$
   if 
    \footnote{We note that this rate is different from the one in 
    the submitted version of the paper. }
    $
       T \geq \log 
       \left(
           \log \left(\frac{1}{\epsilon}\right) 
           /
           \log \left(\frac{2\lambda^2}{\Lambda\|L(u_0) - f\|_{\ll}}\right)
        \right).
    $
\end{lemma}

\begin{proof}[Proof of Lemma~\ref{lemma:fast_convergence}]
    Re-writing the updates in Lemma~\ref{lemma:fast_convergence} as,
    \begin{align}
        u_{t+1} &= u_t + p_t \\
        L'(u_t) p_t &= -(L(u_t) - f) \numberthis \label{eq:Lprime}
    \end{align}
    Now, upper bounding $L(u_{t+1}) - f$ for all $x \in \Omega$
    we have,
    \begin{align*}
        & L(u_{t+1}(x)) - f(x)  \\
            &= L(u_t(x)) - f(x)
            + \int_{0}^1\left(L'(u_t(x) + t(u_{t+1}(x) - u_t(x)))\right)(u_{t+1}(x) - u_t(x))dt\\
            &= L(u_t(x)) - f(x)  + L'(u_t(x))p_t(x)
            + \int_{0}^1\left(L'(u_t(x) + t(u_{t+1}(x) - u_t(x))) - L'(u_t(x))\right)p_t(x)dt\\
            &= 
            \int_{0}^1\left(L'(u_t(x) + t(u_{t+1}(x) - u_t(x))) - L'(u_t(x))\right)p_t(x)dt
    \end{align*}    
    where we use ~\eqref{eq:Lprime} in the final step.
    
    Taking $\ll$ norm on both sides and using the fact that
    $\|L'(u) - L'(v)\|_{\ll} \leq \Lambda \|u - v\|_{\ll}$, 
    we have 
    \begin{align*}
        \|L(u_{t+1}) - f\|_{\ll}
        \leq \int_0^1
        \Lambda t \|u_{t+1} - u_t\|_{\ll}\|p_t\|_{\ll} dt
    \end{align*}
    
    Noting that for all $x \in \Omega$, we have
    $u_{t+1} - u_t = p_t$, 
    and using the fact that for all $u,v$ 
    $\|L'(u)^{-1}v\|_{\ll}\leq\frac{1}{\lambda}\|v\|_{\ll}$
    we have,
    $\|L'(u_t)p_t\|_{\ll} \leq \frac{1}{\lambda} \|p_t\|_{\ll}$
    \begin{align*}
         \|L(u_{t+1}) - f\|_{\ll}
         &\leq
            \int_{0}^1 \Lambda t\|u_{t+1} - u\|_{\ll}\|p_t\|_{\ll} dt\\
        &\leq \Lambda/2 \|p_t\|_{\ll}^2\\
        &\leq \Lambda/2 \|-L'(u_t)^{-1} ( L(u_t) - f)\|_{\ll}^2\\
        &\leq \frac{\Lambda}{2\lambda^2}\|L(u_t) - f)\|_{\ll}^2
    \end{align*}    
    where we use the result from Lemma~\ref{lemma:frechet_upperbound}
    in the last step.

    Therefore we have
    \begin{align*}
         \|L(u_{t+1}) - f\|_{\ll}
         &\leq
         \left(\frac{\Lambda}{2\lambda^2}\right)^{2^t - 1}
         \left(L(u_0) - f\right)^{2^t}\\
         \implies
         \|L(u_{t+1}) - f\|_{\ll}
         &\leq
         \left(\frac{\Lambda}{2\lambda^2}\right)^{2^t - 1}
         \left(L(u_0) - L(u^\star)\right)^{2^t}\\
         \implies
         \|u_{t+1} - u^\star\|_{\ll}
         &\leq
         \frac{1}{\lambda}\left(\frac{\Lambda}{2\lambda^2}\right)^{2^t - 1}
         \left\|L(u_0) - L(u^\star)\right\|_{\ll}^{2^t}.
    \end{align*}
    Therefore, if 
    $$\frac{\Lambda}{2\lambda^2}\|L(u_0) - L(u^\star)\|_{\ll}\leq 1,$$
    then we have 
    \begin{align*}
         \|u_{t+1} - u^\star\|_{\ll}
         \leq \epsilon,
    \end{align*}
    for
   \begin{align*}
       T \geq \log 
       \left(
           \log \left(\frac{1}{\epsilon}\right) 
           /
           \log \left(\frac{2\lambda^2}{\Lambda\|L(u_0) - f\|_{\ll}}\right)
        \right).
   \end{align*}
\end{proof}

\section{Additional experimental results} 
\label{sec:additional-experiments}

We provide additional results for Navier-Stokes equation for noisy inputs and observations in \cref{table:results-navier-stokes-visc-0.001-nl-0.004} and \cref{table:results-navier-stokes-visc-0.01-nl-0.004}. For these experiments, the maximum variance of Gaussian noise added to inputs and observations is 0.004. We observe that weight-tied FNO and FNO-DEQ outperform non-weight-tied architectures.

\begin{table*}[th!]
\centering
\resizebox{0.9\textwidth}{!}{%
\begin{tabular}{cccccc}
\toprule
\multirow{3}{*}{Architecture} & \multirow{3}{*}{Parameters} & \multirow{3}{*}{\#Blocks} & \multicolumn{3}{c}{Test error $\downarrow$} \\
\cmidrule(lr){4-6}
& & & $\sigma^2_{\max}=0$ & $(\sigma^2_{\max})^i=0.004$ & $(\sigma^2_{\max})^t=0.004$ \\
\midrule
FNO & 2.37M & 1 & 0.184 $\pm$ 0.002 & 0.238 $\pm$ 0.008 & 0.179 $\pm$ 0.004\\
FNO & 4.15M & 2 & 0.162 $\pm$ 0.024 & 0.196 $\pm$ 0.011 & 0.151 $\pm$ 0.010\\
FNO & 7.71M & 4 & 0.157 $\pm$ 0.012 & 0.216 $\pm$ 0.002 & 0.158 $\pm$ 0.009\\
\midrule
FNO++ & 2.37M & 1 & 0.199 $\pm$ 0.001 & 0.255 $\pm$ 0.002 & 0.197 $\pm$ 0.004 \\
FNO++ & 4.15M & 2 & 0.154 $\pm$ 0.005 & 0.188 $\pm$ 0.006 & 0.157 $\pm$ 0.006 \\
FNO++ & 7.71M & 4 & 0.151 $\pm$ 0.003 & 0.184 $\pm$ 0.008 & 0.147 $\pm$ 0.004\\
\midrule
FNO-WT & 2.37M & 1 & 0.151 $\pm$ 0.007 & 0.183 $\pm$ 0.026 & 0.129 $\pm$ 0.018 \\
FNO-DEQ & 2.37M &  1 & \textbf{0.128 $\pm$ 0.004} & \textbf{0.159 $\pm$ 0.005} & \textbf{0.121 $\pm$ 0.015} \\
\bottomrule
\end{tabular}}
\caption{\textbf{Results on incompressible Steady-State Navier-Stokes (viscosity=0.001)}: clean data (Col 4), noisy inputs (Col 5) and noisy observations (Col 6) with max variance of added noise being $(\sigma^2_{\max})^i$ and $(\sigma^2_{\max})^t$, respectively. Reported test error has been averaged on three different runs with seeds 0, 1, and 2. \\$\ddagger$ indicates that the network diverges during training for one of the seeds.}\label{table:results-navier-stokes-visc-0.001-nl-0.004}
\end{table*}

\begin{table*}[th!]
\centering
\resizebox{0.9\textwidth}{!}{%
\begin{tabular}{cccccc}
\toprule
\multirow{3}{*}{Architecture} & \multirow{3}{*}{Parameters} & \multirow{3}{*}{\#Blocks} & \multicolumn{3}{c}{Test error $\downarrow$} \\
\cmidrule(lr){4-6}
& & & $\sigma^2_{\max}=0$ & $(\sigma^2_{\max})^i=0.004$ & $(\sigma^2_{\max})^t=0.004$ \\
\midrule
FNO & 2.37M & 1 & 0.181 $\pm$ 0.005 & 0.207 $\pm$ 0.003 & 0.178 $\pm$ 0.008 \\
FNO & 4.15M & 2 & 0.138 $\pm$ 0.007 & 0.163 $\pm$ 0.003 & 0.137 $\pm$ 0.006 \\
FNO & 7.71M & 4 & 0.152 $\pm$ 0.006 &  0.203 $\pm$ 0.055 & 0.151 $\pm$ 0.008 \\
\midrule
FNO++ & 2.37M & 1 & 0.188 $\pm$ 0.002 & 0.217 $\pm$ 0.001 & 0.187 $\pm$ 0.005 \\
FNO++ & 4.15M & 2 & 0.139 $\pm$ 0.004 & 0.170 $\pm$ 0.005 & 0.138 $\pm$ 0.005 \\
FNO++ & 7.71M & 4 & 0.130 $\pm$ 0.005 & 0.168 $\pm$ 0.007 & 0.126 $\pm$ 0.007 \\
\midrule
FNO-WT & 2.37M & 1 & 0.099 $\pm$ 0.007 & 0.159 $\pm$ 0.029 & 0.123 $\pm$ 0.023 \\
FNO-DEQ & 2.37M & 1 & \textbf{0.088 $\pm$ 0.006} & \textbf{0.104 $\pm$ 0.001} & \textbf{0.116 $\pm$ 0.005} \\
\bottomrule
\end{tabular}}
\caption{\textbf{Results on incompressible Steady-State Navier-Stokes (viscosity=0.01)}: clean data (Col 4), noisy inputs (Col 5) and noisy observations (Col 6) with max variance of added noise being $(\sigma^2_{\max})^i$ and $(\sigma^2_{\max})^t$, respectively. Reported test error has been averaged on three different runs with seeds 0, 1, and 2. \\$\ddagger$ indicates that the network diverges during training for one of the seeds.} \label{table:results-navier-stokes-visc-0.01-nl-0.004}
\end{table*}

\paragraph{Convergence analysis of fixed point.} We report variations in test error, absolute residual $\| G_\theta(\rvz_t) - \rvz_t\|_2$, and relative residual $\frac{\| G_\theta(\rvz_t) - \rvz_t\|_2}{\| \rvz_t \|_2}$ with an increase in the number of solver steps while solving for the fixed point in FNO-DEQ, for both Darcy Flow (See~\cref{tab:darcy_flow_convergence}) and Steady-State Navier Stokes (See~\cref{tab:ns_0.01_convergence}). We observe that all these values decrease with increase in the number of fixed point solver iterations and eventually saturate once we have a reasonable estimate of the fixed point. 
We observe that increasing the number of fixed point solver iterations results in a better estimation of the fixed point. 
For steady state PDEs, we expect the test error to reduce as the estimation of the fixed point improves. 
Furthermore, at inference time we observe that the test error improves (i.e. reduces) with increase in the number of fixed point solver iterations even though the FNO-DEQ is trained with fewer solver steps.
For Navier-Stokes with viscosity 0.01, at inference time we get a test MSE loss of 0.0744 with 48 solver steps from 0.0847 when used with 24 solver steps.

This further bolsters the benefits of DEQs (and weight-tied architectures in general) for training neural operators for steady-state PDEs. Moreover, performance saturates after a certain point once we have a reasonable estimate of the fixed point, hence showing that more solver steps stabilize to the same solution.

\begin{table}[th!]
    \centering
    \begin{tabular}{cccc}
    \toprule
    Solver steps & Absolute residual $\downarrow$  & Relative residual $\downarrow$  & Test Error $\downarrow$ \\
    \midrule
    2 & 212.86 & 0.8533 & 0.0777 \\
    4 & 18.166 & 0.0878 & 0.0269 \\
    8 & 0.3530 & 0.00166 & 0.00567 \\
    16 & 0.00239 & 1.13e-5 & 0.00566 \\
    32 & 0.000234 & 1.1e-6 & 0.00566 \\
    \bottomrule \\
    \end{tabular}
    \caption{Convergence analysis of fixed point for noiseless Darcy Flow: The test error, absolute residual $\| G_\theta(\rvz_t) - \rvz_t\|_2$ and relative residual $\frac{\| G_\theta(\rvz_t) - \rvz_t\|_2}{\| \rvz_t \|_2}$ decrease with increase in the number of fixed point solver iterations. The performance saturates after a certain point once we have a reasonable estimate of the fixed point. We consider the noiseless case, where we do not add any noise to inputs or targets. }
    \label{tab:darcy_flow_convergence}
\end{table}

\begin{table}[th!]
    \centering
    \begin{tabular}{cccc}
    \toprule
    Solver steps & Absolute residual $\downarrow$  & Relative residual $\downarrow$ & Test Error $\downarrow$ \\
    \midrule
    4 & 544.16 & 0.542 & 0.926 \\
    8 & 397.75 & 0.408 & 0.515 \\
    16 & 150.33 & 0.157 & 0.147 \\
    24 & 37.671 & 0.0396 & 0.0847 \\
    48 & 5.625 & 0.0059 & 0.0744 \\
    64 & 3.3 & 0.0034 & 0.0746 \\
    \bottomrule
    \end{tabular}
    \caption{Convergence analysis of fixed point for noiseless incompressible Steady-State Navier-Stokes with viscosity=0.01: The test error, absolute residual $\| G_\theta(\rvz_t) - \rvz_t\|_2$ and relative residual $\frac{\| G_\theta(\rvz_t) - \rvz_t\|_2}{\| \rvz_t \|_2}$ decrease with increase in the number of fixed point solver iterations. The performance saturates after a certain point once we have a reasonable estimate of the fixed point. We consider the noiseless case, where we do not add any noise to inputs or targets.}
    \label{tab:ns_0.01_convergence}
\end{table}

\begin{figure}[!t]
    \centering
    \begin{subfigure}[b]{0.7\textwidth}
        \centering
        \includegraphics[width=1\linewidth]{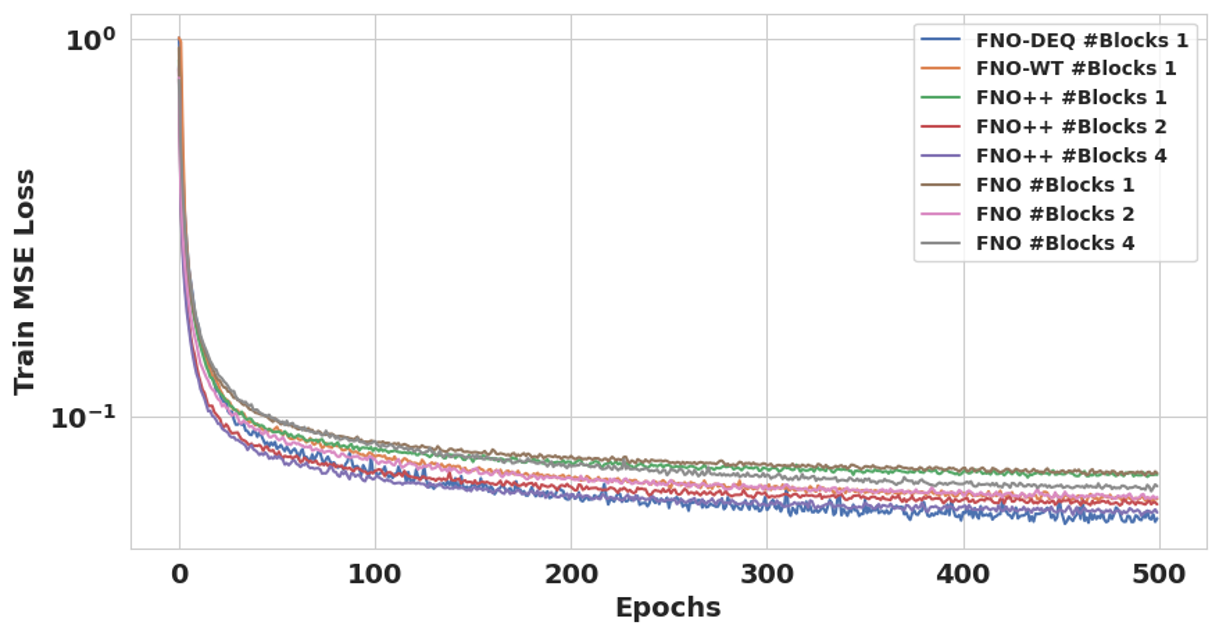}
        \caption{Training Loss Curve}
    \end{subfigure}
    \hfill
    \begin{subfigure}[b]{0.7\textwidth}
        \centering
        \includegraphics[width=1\linewidth]{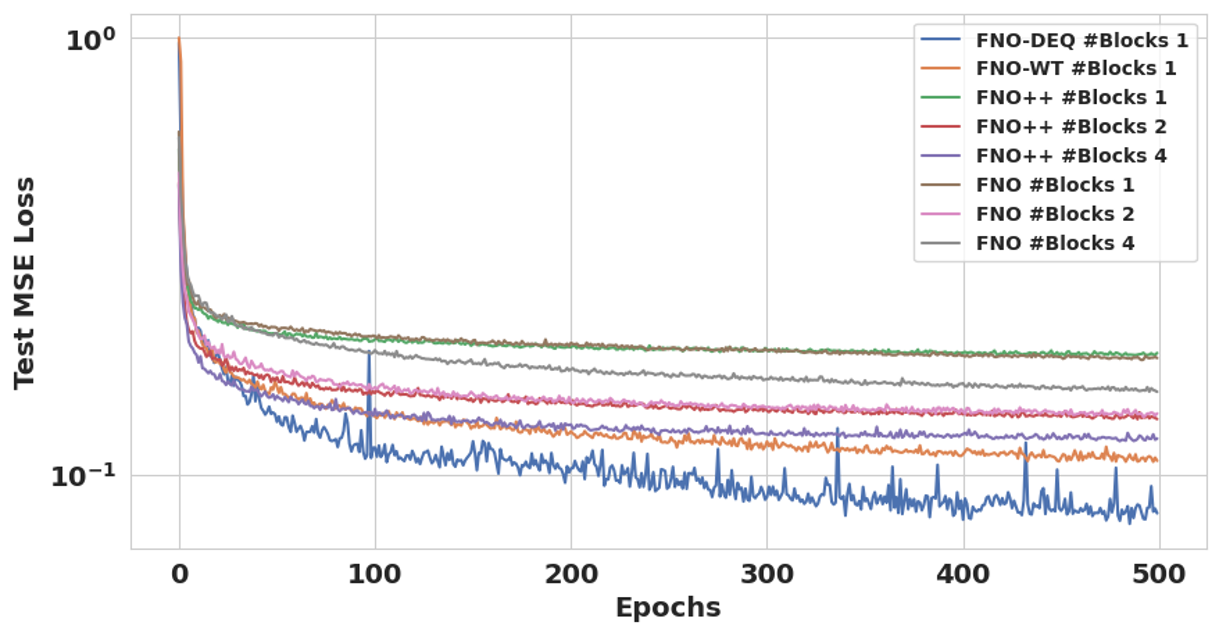}
        \caption{Test Loss Curve}
        \end{subfigure}
    \caption{Training and Test Loss Curves for Steady-State Navier-Stokes with viscosity $0.01$. The $x$ axis is the number of epochs and $y$ axis is the MSE loss in $\log$ scale. 
    Note that while all the models converge to approximately the same MSE loss value while training, DEQs and weight-tied networks get a better test loss in fewer epochs.
    }
    \label{fig:train_val_error}
\end{figure}

\begin{table*}[th!]
\centering
\resizebox{0.95\textwidth}{!}{%
\begin{tabular}{cccccc}
\toprule
\multirow{3}{*}{Architecture} & \multirow{3}{*}{Parameters} & \multirow{3}{*}{\#Blocks} & \multicolumn{3}{c}{Test error $\downarrow$} \\
\cmidrule(lr){4-6}
& & & $\sigma^2_{\max}=0$ & $(\sigma^2_{\max})^i=0.001$ & $(\sigma^2_{\max})^t=0.001$ \\
\midrule
FNO & 2.37M & 1 & 0.0080 $\pm$ 5e-4 & 0.0079 $\pm$ 2e-4  &  0.0125 $\pm$ 4e-5 \\
FNO & 4.15M & 2 & 0.0105 $\pm$ 6e-4 & 0.0106 $\pm$ 4e-4 & 0.0136 $\pm$ 2e-5 \\
FNO & 7.71M & 4 & 0.2550 $\pm$ 2e-8 & 0.2557 $\pm$ 8e-9 & 0.2617 $\pm$ 2e-9 \\
\midrule
FNO++ & 2.37M & 1 & 0.0075 $\pm$ 2e-4 & 0.0075 $\pm$ 2e-4 &  0.0145 $\pm$ 7e-4 \\
FNO++ & 4.15M & 2 & 0.0065 $\pm$ 2e-4 & 0.0065 $\pm$ 9e-5 & 0.0117 $\pm$ 5e-5 \\
FNO++ & 7.71M & 4 & 0.0064 $\pm$ 2e-4 & 0.0064 $\pm$ 2e-4 & \textbf{0.0109 $\pm$ 5e-4}  \\
\midrule
FNO-WT & 2.37M & 1 & \textbf{0.0055 $\pm$ 1e-4} & \textbf{0.0056 $\pm$ 5e-5} & 0.0112 $\pm$ 4e-4 \\
FNO-DEQ & 2.37M &  1 & \textbf{0.0055 $\pm$ 1e-4} & \textbf{0.0056 $\pm$ 7e-5} & 0.0112 $\pm$ 4e-4 \\
\bottomrule
\end{tabular}}
\caption{Results on Darcy flow: clean data (Col 4),noisy inputs (Col 5) and noisy observations (Col 6) with max variance of added noise being $(\sigma^2_{\max})^i$ and $(\sigma^2_{\max})^t$, respectively. Reported test error has been averaged on three different runs with seeds 0, 1, and 2.
Here, S-FNO++, S-FNO-WT and S-FNO-DEQ are shallow versions 
of FNO++, FNO-WT and FNO-DEQ respectively.
}
\label{table:results-darcy-flow-all}
\end{table*}
\vspace{-2mm}

\end{document}